%% file: main.tex
\documentclass{article} % For LaTeX2e
\usepackage{arxiv}

\usepackage{amsmath}
\usepackage{amssymb}
\usepackage{mathtools}
\usepackage{amsthm}

\usepackage[ruled,vlined]{algorithm2e}
\SetAlCapNameFnt{\footnotesize}
\SetAlCapFnt{\footnotesize}
\SetAlgorithmName{Alg.}{Alg.}{List of Algorithms}
\usepackage{caption}
\theoremstyle{plain}
\newtheorem{theorem}{Theorem}[section]
\newtheorem{proposition}[theorem]{Proposition}
\newtheorem{lemma}[theorem]{Lemma}
\newtheorem{corollary}[theorem]{Corollary}
\theoremstyle{definition}
\newtheorem{definition}[theorem]{Definition}
\newtheorem{assumption}[theorem]{Assumption}
\theoremstyle{remark}
\newtheorem{remark}[theorem]{Remark}

\usepackage[textsize=tiny]{todonotes}
\usepackage{microtype}
\usepackage{graphicx}
\usepackage{booktabs} % for professional tables

\usepackage{enumitem}
\setlist[itemize]{align=parleft,left=0pt..1em}
\setlist[enumerate]{align=parleft,left=0pt..1em}

\usepackage{hyperref}
\usepackage[capitalize,noabbrev]{cleveref}

\newcommand\comment[1]{\ignorespaces}

\usepackage[utf8]{inputenc} % allow utf-8 input
\usepackage[T1]{fontenc}    % use 8-bit T1 fonts
\usepackage{hyperref}       % hyperlinks
\usepackage{url}            % simple URL typesetting
\usepackage{booktabs}       % professional-quality tables
\usepackage{amsfonts}       % blackboard math symbols
\usepackage{nicefrac}       % compact symbols for 1/2, etc.
\usepackage{microtype}      % microtypography

\title{A Local Graph Limits Perspective on Sampling-Based GNNs}

\author{%
Yeganeh Alimohammadi\\
Stanford University\\
\texttt{yeganeh@stanford.edu} \And
Luana Ruiz\\
Johns Hopkins University\\
\texttt{lrubini1@jh.edu} \And  Amin Saberi\\
Stanford University \\
\texttt{saberi@stanford.edu}
}

\definecolor{darkblue}{rgb}{0.0,0.0,0.65}
\definecolor{darkred}{rgb}{0.68,0.05,0.0}
\definecolor{darkgreen}{rgb}{0.0,0.29,0.29}
\definecolor{darkpurple}{rgb}{0.47,0.09,0.29}
\usepackage{hyperref}
\hypersetup{
   colorlinks = true,
   citecolor  = darkblue,
   linkcolor  = darkred,
   filecolor  = darkblue,
   urlcolor   = darkblue,
 }

\begin{document}
\maketitle
\begin{abstract}
We propose a theoretical framework for training Graph Neural Networks (GNNs) on \emph{large} input graphs via training on \emph{small}, \emph{fixed-size} sampled subgraphs. This framework is applicable to a wide range of models, including popular sampling-based GNNs, such as GraphSAGE and  FastGCN. 
Leveraging the theory of graph local limits, we prove that, under mild assumptions, parameters learned from training sampling-based GNNs on small samples of a large input graph are within an $\epsilon$-neighborhood of the outcome of training the same architecture on the  {\em whole graph}. We derive bounds on the number of samples, the size of the graph, and the training steps required as a function of $\epsilon$. % comment: critically our bounds are independent of the input graph size 
Our results give a novel theoretical understanding for using sampling in training GNNs. 
% comment it's unclear what they refer to
They also suggest that by training GNNs on small samples of the input graph, practitioners can identify and select the best models, hyperparameters, and sampling algorithms  more efficiently. We empirically illustrate our results on a node classification task on large citation graphs, observing that sampling-based GNNs trained on local subgraphs 12$\times$ smaller than the original graph achieve comparable performance to those trained on the input graph.
\end{abstract}

\section{Introduction} \label{sec: intro}

As the size and complexity of graph data continue to increase, there is a growing need to find ways to scale Graph Neural Networks (GNNs). Yet, scaling GNNs to larger graphs faces two key obstacles: training inefficiencies due to repeated gradient calculations at every node; and large memory requirements for storing not only the graph but also the node embeddings. %two major challenges arise when trying to scale GNNs to large graphs: the inefficiency of training by repeatedly calculating gradients at each node in the graph; and the need for significant memory to store the entire adjacency matrix and corresponding node embeddings. 
To overcome these challenges, a variety of efficient GNN training algorithms have been introduced which leverage a wide array of sampling techniques. Examples include GraphSAGE \citep{hamilton2017inductive}, FastGCN \citep{chen2018fastgcn}, and shaDoW-GNN \citep{zeng2021decoupling},
among others (see \citep{chapter6} for a thorough review). 

%{\bf Amin:} I am a little worried   the next paragraph gives the impression to the reader that our paper gives an address to the questions mentioned there. We don't directly answer that. We just say  you can run the search on smaller samples. Maybe it is better to initially motivate the result by saying despite the success there is no theoretical justification why training and sampling on subgraphs should help at all. Another version of this paragraph can show up right before the summary where we say our results enables practitioners to make these choices by experimenting on small subgraphs.

 %Although these techniques have shown promising results in reducing computational complexity while retaining good performance, practitioners still face many questions when
 %designing and optimizing a GNN
 %for a specific learning task: Which method should we use to sample the computational graph? How many nodes to sample for estimating the gradient? How to set the number of layers of the GNN and other hyperparameters of the model? Unfortunately, these choices are typically made by trial and error, which is slow and inefficient, and lacks formal guarantees. As a result, selecting the best sampling algorithm and architecture and tuning their hyperparameters is often a burdensome task. 

% comment : what they refers to
% comment: computation complexity is redundant in the following paragraph
Despite the success of sampling-based GNNs in practice, there are still no formal models that explain why they can reduce the computation and memory requirements of GNNs without significantly compromising their performance\footnote{We use the term `sampling-based GNN' broadly to refer to any GNN architecture that utilizes node sampling and/or computational graph sampling.}. 
%This paper takes the first step in this direction.% by proposing a theoretical model and analysis of sampling-based GNNs. Specifically, we propose a  framework that abstracts away  the details of specific architectures and sampling algorithms and defines a general class of sampling-based GNNs (see Algorithm \ref{alg: sampling-GNN}). 
% comment: the following sentence is too long maybe remove 'abstract aways,,,' part
We take a step in this direction by proposing a theoretical framework that abstracts away  the details of specific architectures and sampling algorithms and defines a general class of sampling-based GNNs (see Algorithm \ref{alg: sampling-GNN}). In this framework, sampling is employed in two steps. First, nodes are sampled to estimate the gradient and update the GNN weights for the next iteration (\emph{node sampling}). Second, a sampling method is used to prune the computational graph when computing the node embeddings via neighborhood aggregations (\emph{computational graph sampling}).

%. We call models fitting under this framework \emph{sampling-based GNNs}
% Any GNN architecture (e.g., GCN, GIN \citep{xu2018GIN}, SGC \citep{wu2019simplifying}, etc.) can be trained using these sampling algorithms.} 
As a second contribution, we propose to use this general framework to study a training procedure in which, instead of training the GNN on the whole graph, we train it on a collection of small subgraphs sampled from the input graph. Intuitively, this approximates both node sampling---as gradients are only computed for nodes in the subgraph---and computational graph sampling---as the computational subgraphs are not necessarily induced subgraphs\footnote{Some papers in the early GNN literature introduced both a novel architecture and a sampling-based training algorithm. This has caused confusion as it is not uncommon for the architecture and the algorithm to be called by the same name, e.g., GraphSAGE can refer to both the architecture and the computational graph sampling scheme proposed by \citet{hamilton2017inductive}. In general, when using these names, we will be referring to the sampling algorithm. We will explicitly specify it when referring to the architecture.}.
%Then, we build upon this framework by proposing an algorithm that, instead of training the sampling-based GNN on the large input graph, trains it on small subgraphs sampled from the input graph (Algorithm \ref{alg: training}). 
%The motivation behind this approach is that
%GNNs produce similar embeddings on subgraphs with similar local structures, which can be explained by their permutation equivariance \citep{haggai2019invariant,keriven2019universal} and stability properties \citep{gama19-stability,gama2019stabilityscat}.
%Permutation equivariance means that if we permute or relabel the nodes at the input of the GNN, the output is permuted in the same way, and stability can be seen as a relaxation of this property to approximate permutation sysmmetries. As a result, GNNs learn similar embeddings for nodes in different regions of the graph as long as they have similar local structures.
We prove the validity of this approximation theoretically, by showing that training such sampling-based GNNs on small sampled subgraphs yields a similar outcome to training these GNNs directly on the large target graph (Theorem \ref{thm: stochastic GNN}). %which can in turn be interpreted to mean that training sampling-based GNNs yields a similar outcome to training on the large input graph.
The proof relies on 
% To formalize how different graphs can be part of the same family, 
% we turn to
the theory of graph limits \citep{Aldous2004, benjamini2001}, which provides a way to understand the behavior of a family of graphs with similar local structures. We view the large input graph as the infinite `limit' of sampled graphs with similar local structures (i.e., similar motifs such as triangles and $k$-cycles). Then, we show that the training behavior on these smaller subgraphs converges to the training behavior on the limit graph. 
%Specifically, for any  $\epsilon>0$, we prove the existence of a constant $N_\epsilon$ that is independent of the size of the input graph such that training a sampling-based GNN on samples of size $N_\epsilon$ or larger yields convergence to an $\epsilon$-neighborhood of the GNN trained in the limit graph. %Recall that the limit graph approximated the behavior of the large input graph. 
%Moreover, our analysis shows that this convergence occurs in $\mathcal{O}(1/\epsilon^2)$ training steps.

In practice, our results apply to a variety of architectures, including GCN \citep{kipf17-classifgcnn}, GraphSAGE, and GIN with mean readout \citep{xu2018GIN}; and a variety of gradient and computational graph sampling schemes, including the neighborhood sampling scheme from GraphSAGE, FastGCN, and shaDoWGNN \citep{zeng2021decoupling} (Corollary \ref{cor: application} and Theorem \ref{thm: transductive}). Therefore, our results significantly extend previous results on GNN convergence, which focused almost exclusively on convolutional GNNs \citep{ruiz20-transf,keriven2020convergence,roddenberry2022local}. Our theoretical findings also provide a practical guideline for training GNNs on large graphs more efficiently. They suggest that 
practitioners can compare different sampling-based GNNs and their choice of hyperparameters on small samples of the input graph, knowing that these small samples are good approximations for training on the entire input graph as long as their underlying models and architectures fit within our framework.

We demonstrate the validity of our results empirically on two node classification tasks: on large citation graphs with up to 20,000 nodes (Section \ref{sec: experiment}); an on a \emph{very} large citation network from the ogbn-mag graph, with 200,000 nodes. We observe that various sampling-based GNNs trained on local subgraphs achieve comparable performance to those trained on the large input graph, even when the subgraphs are up to 40$\times$ smaller than the original graph.

%We expect that even small samples from the large graph should capture sufficient local information, while improving the computational cost of training the model. In fact, 
%our main theorem (Theorem \red{Z}) shows that training a sampling-based GNN on small samples of fixed size yields similar results to training the same GNN in the limit.

%choosing the appropriate sampling-based GNN for a learning task on a large graph. In particular, they allow the practitioner to efficiently scrutinize and compare different sampling-based GNNs by training them with the hyperparameters of their choice on small samples of the input graph, knowing that they will achieve 
%an approximation of the loss that would result approximate results to those that would be obtained by training the same model, with the same sampling algorithm, on the entire input graph.

%\todo{ contributions summarize in 2 or at most 3 things!}
\noindent \textbf{Summary.}~Our main contributions are:
\setlist{nolistsep}
\begin{itemize}[noitemsep]
\item  A unifying framework for sampling-based GNNs (Algorithm \ref{alg: sampling-GNN}) incorporating gradient  
and computational graph sampling. %, which is used to prune the computational graph.
% Building upon this framework we propose  training sampling-based GNN on small subgraphs sampled from the large input instead of training on the entire input  (Algorithm \ref{alg: training}).
%\todo{I have mixed feeling about having this as a separate item since it's not a major contribution. Maybe if we combine it with the previous item to have a shorter one it would be less emphasis on it.}

\item A simplified training procedure to approximate sampling-based GNNs, wherein GNNs are trained on small subgraphs sampled from the original input graph.

% next contribution is too long.
\item Proving, using the theory of graph limits, that training GNNs on small samples of the large input graph yields learned parameters within an $\epsilon$-neighborhood of training the GNN on the large target graph (Theorems \ref{thm: stochastic GNN} and \ref{thm: transductive}). This can be interpreted to mean that sampling-based GNNs produce similar outcomes as training on the original input graph.
%We derive bounds on the number of samples, the size of the graphs sampled, and the training steps required as a function of $\epsilon$ . 
\item Extending GNN convergence results to many well-known architectures such as GCN, GraphSAGE and GIN, and many node and computational graph sampling schemes such as neighborhood sampling (GraphSAGE), FastGCN, and shaDoW-GCN (Corollary \ref{cor: application} and Theorem \ref{thm: transductive}). 

\item Empirical results on a node classification task on PubMed and ogbn-mag, in which we observe that sampling-based GNNs trained on local subgraphs achieve comparable performance to those trained on the large target graph (Section \ref{sec: experiment}).
\end{itemize}
\raggedbottom
\section{Related Work}\label{sec: relwork}
\noindent \textbf{Sampling-based GNNs.}~Stochastic node sampling for GNN training is inspired by minibatch stochastic gradient descent and was first proposed by \citet{hamilton2017inductive}. \citet{hamilton2017inductive} also introduced GraphSAGE, which can be trained via random neighborhood sampling, a type of computational graph sampling detailed in Appendix \ref{sec:app}. 
Aiming to further decrease complexity, FastGCN uses importance sampling to sample the computational graph \citep{chen2018fastgcn}, shaDoW-GCN first samples a collection of subgraphs of depth $K$---instead of sampling during training---and then trains a conventional GNN on them \citep{zeng2021decoupling}. Our framework also applies to more recent advancements on sampling-based GNNs such as GNNAutoScale \citep{fey2021gnnautoscale}, GraphFM \citep{yu2022graphfm}, LMC \citep{shi2022lmc}, and IBMB \citep{gasteiger2022IBMB}.
These architectures and sampling mechanisms are discussed in further detail in Appendix \ref{sec:app}. See also \citep{liu2021sampling} for a complete survey on sampling-based GNNs.

A different related line of work is neural network pruning at the training stage, which is a technique for reducing the computational complexity and memory requirements of deep neural networks \citep{zhu2017prune,gale2019state,strubell2019energy}. In this line of work, they make training more effective and efficient by sampling the connections between neurons of a neural network (while in our work we study sample data that has a network structure).
Recent works on the lottery ticket hypothesis \citep{frankle2018lottery,frankle2020linear} have shed light on the possibility of sampling subnetworks during the early phases of training that achieve comparable accuracy to the full network while significantly reducing its complexity. 

Tangential to our work is \citep{yehudai2021local} explores distribution shifts in graph structures, emphasizing potential generalization issues. In contrast, we focus on unifying frameworks for training sampling-based GNNs and aim to prove the convergence of training across graph families with consistent local structures.

% Tangential to neural network pruning, our work explores the application of sampling techniques on large input graphs, leveraging their underlying network structure to reduce the data size while maintaining the performance of graph neural networks.

\comment{Graph limits have been used in the past to study functions on graphs and networks, with a focus on analyzing convergence behavior and robustness to misspecification. For dense, exchangeable networks, graphons are well-studied objects\citep{}. In particular, graphons have also been used in the context of transferability results on graph neural networks \citep{ruiz20-transf}, where they provide a theoretical framework for understanding how GNNs trained on small graphs can be deployed on larger graphs with a guarantee of performance.

In the case of sparse graphs, 
}
\noindent \textbf{Benjamini-Schramm convergence.}~The theory of graph limits introduced by \citet{benjamini2001, Aldous2004} has been used for studying random network models.  
Almost all sparse random graph models, including Erd\"os-R\'enyi graphs \citep[Theorem 2.17]{RemcoVol2}, configuration models \citep{dembo2010ising} \citep[Theorem 4.5]{RemcoVol2}, preferential attachment models \citep{berger2014}, geometric random graphs \citep{BollobasJansonRiordon}, and motif-based models  \citep{alimohammadi2022algorithms} are known to have graph limits. 

\input{gcn_sampling}

\input{gcn_limit}
\input{convergence_sampling_gcn}

%%%%%%%%%%%%%%%%%%%%%%%%%%%%%%%%%%%%%%%%%%%%%%%%%%%
%%%%%%%%%%%%%%%%%%% EXPERIMENTS %%%%%%%%%%%%%%%%%%%
%%%%%%%%%%%%%%%%%%%%%%%%%%%%%%%%%%%%%%%%%%%%%%%%%%%

\section{Experiments}\label{sec: experiment}

We validate our results empirically through two sets of experiments: an ablation study of node and computational graph sampling on a large citation network (PubMed, $\sim$20k nodes); and a more realistic example on a very large citation network (sample from ogbn-mag, 200k nodes). Dataset details are provided in Appendix \ref{appendix: experiment_details}. Common hyperparameters and training details are described in Appendix \ref{appendix: experiment_details}, and details specific to each experiment are listed in the corresponding sections.
% \todo{Isn't graph sampling already captured when any of the sampling-based GNN architectures are implemented? For example GraphSAGE? L: No... these are implemented separately in PyG, and also defined separately in \citep{hamilton2017inductive}. I.e., this paper has two contributions: the GraphSAGE architecture (the novelty of the model is the concatenation of neighbor embeddings) and neighborhood sampling. GraphSAGE can be implemented with or without neighborhood sampling.}

\subsection{Ablation Study on a Large Graph}

In first the experiment in this section, we consider node sampling independently, and then incorporate computational graph sampling in the second experiment.
%, and no sampling. 
In each case, we compare GNNs trained on the full $N$-node graph with the same type of GNN, but trained on a collection of $n$-node subgraphs. These subgraphs are sampled via breadth-first search from random seeds sampled at random from the full graph at regular training intervals $\gamma$ (in epochs). In Figures \ref{fig:node_sampling}--\ref{fig:comp_sampling}, the dashed lines correspond to the best test accuracy of the model trained on the full $N$-node graph. The solid lines are the per epoch test accuracy of the models trained on the collection of stochastic $n$-node subgraphs. 

\begin{figure*}[t]
\centering
\includegraphics[width=0.32\linewidth]{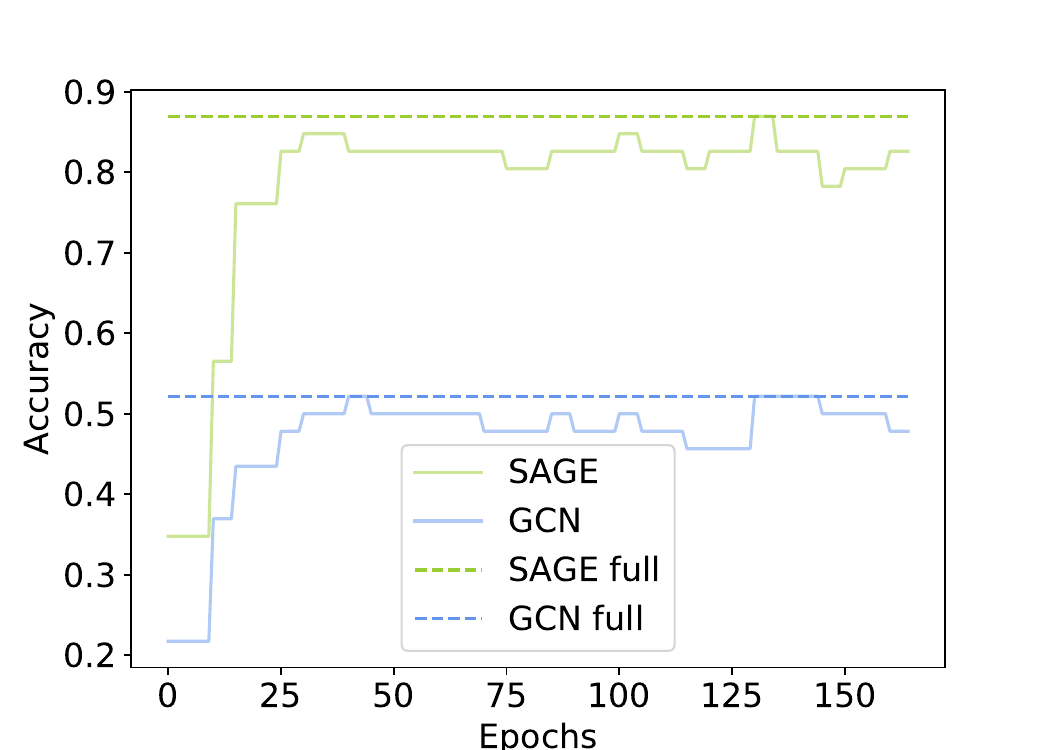}
\includegraphics[width=0.32\linewidth]{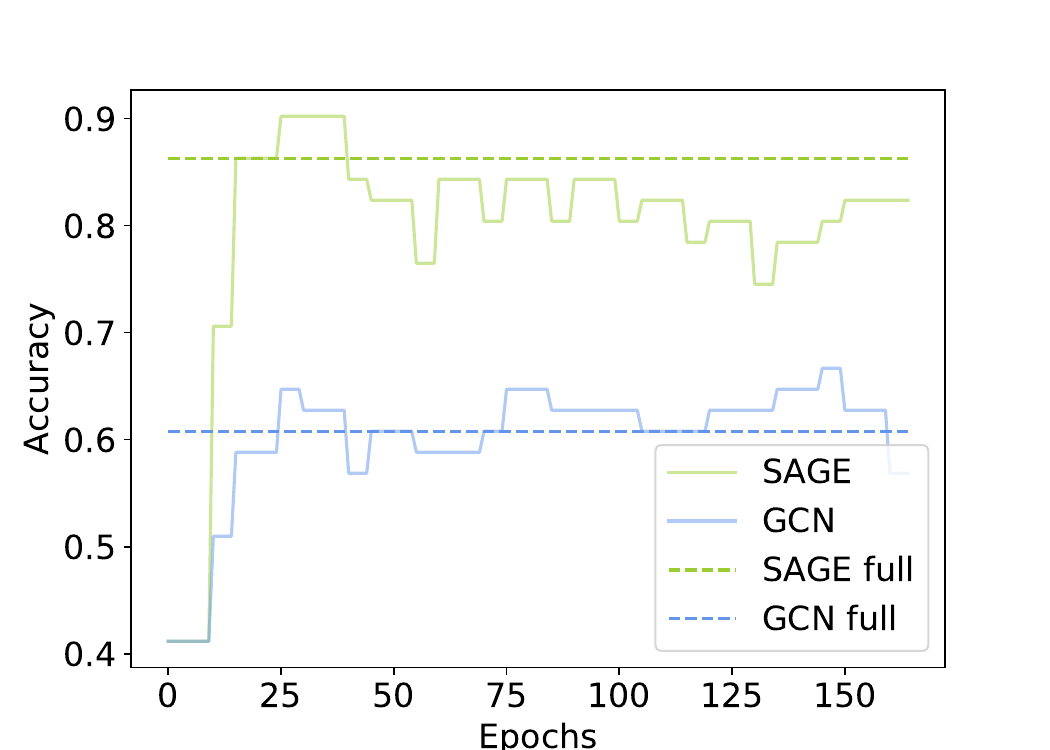}
\includegraphics[width=0.32\linewidth]{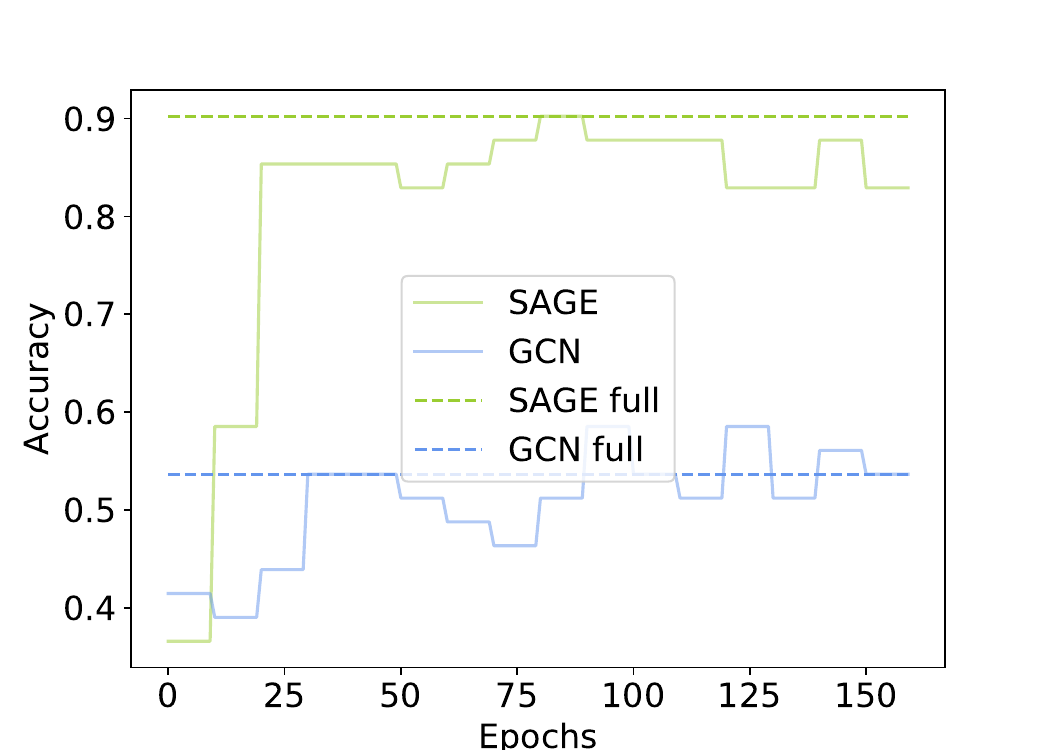}
\caption{Node sampling results for PubMed with batch size $32$. We consider three scenarios in terms of the small graph size $n$ and the graph sampling interval $\gamma$: $n=1500$, $\gamma=15$ epochs (left); $n=2000$, $\gamma=15$ epochs (center); $n=1500$, $\gamma=10$ epochs (right). Note that these graphs have size equal to approximately $10\%$ of the original graph size.}
\label{fig:node_sampling}
\end{figure*}

\noindent \textbf{Node sampling.}~We consider GNNs trained using the uniform random node sampling strategy (SGD) described in Section \ref{sec: sampling GCN} and used, e.g., in \citep{hamilton2017inductive}. This sampling technique consists of partitioning the nodes into batches and, at each step, only considering the nodes in the current batch to compute the gradient updates. 
%We fix the batch size at $256$ and sample half the nodes using negative sampling.
%For Cora and CiteSeer, $n=800$ and the GNNs are trained for $300$ epochs with a learning rate $1e{-4}$ and batch size $32$. For PubMed, $n=3000$
In Figure \ref{fig:node_sampling}, we consider three scenarios in terms of the small graph size $n$ and the graph sampling interval $\gamma$: $n=1500$, $\gamma=15$ epochs (left); $n=2000$, $\gamma=15$ epochs (center); $n=1500$, $\gamma=10$ epochs (right). Note that these graphs have size equal to approximately $10\%$ of the original graph size. The GNNs are trained for $150$ epochs with learning rate $1e{-4}$ and batch size $32$. In Figure \ref{fig:node_sampling}, we observe that the GNNs trained on collections of random subgraphs achieve comparable performance to the GNNs trained on the full graph. Increasing the graph size $n$ leads some improvement for the GCN, but causes GraphSAGE to overfit. Increasing the sampling rate increases the variation in accuracy, but leads to a slight improvement in performance for both architectures.
%, and even outperform them on Cora and CiteSeer. On PubMed, we observe more variance in the per epoch performance, but this is expected since PubMed is a much larger graph (see Table \ref{tab:stats}).

\begin{figure*}[t] 
\centering
\includegraphics[width=0.32\linewidth]{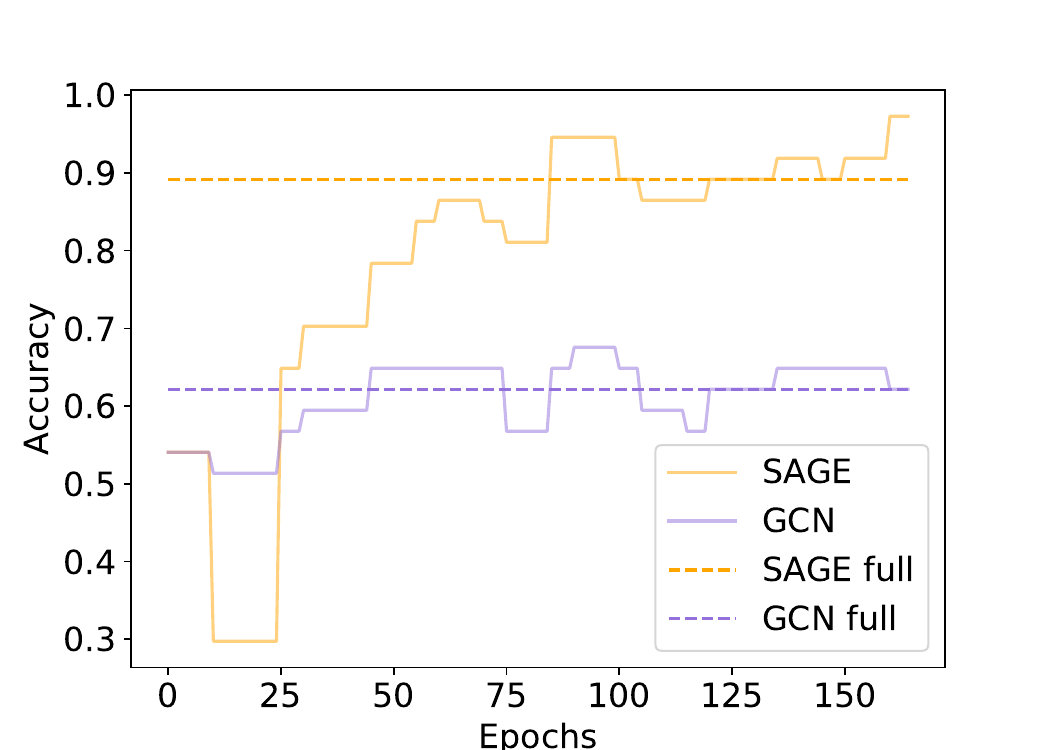}
\includegraphics[width=0.32\linewidth]{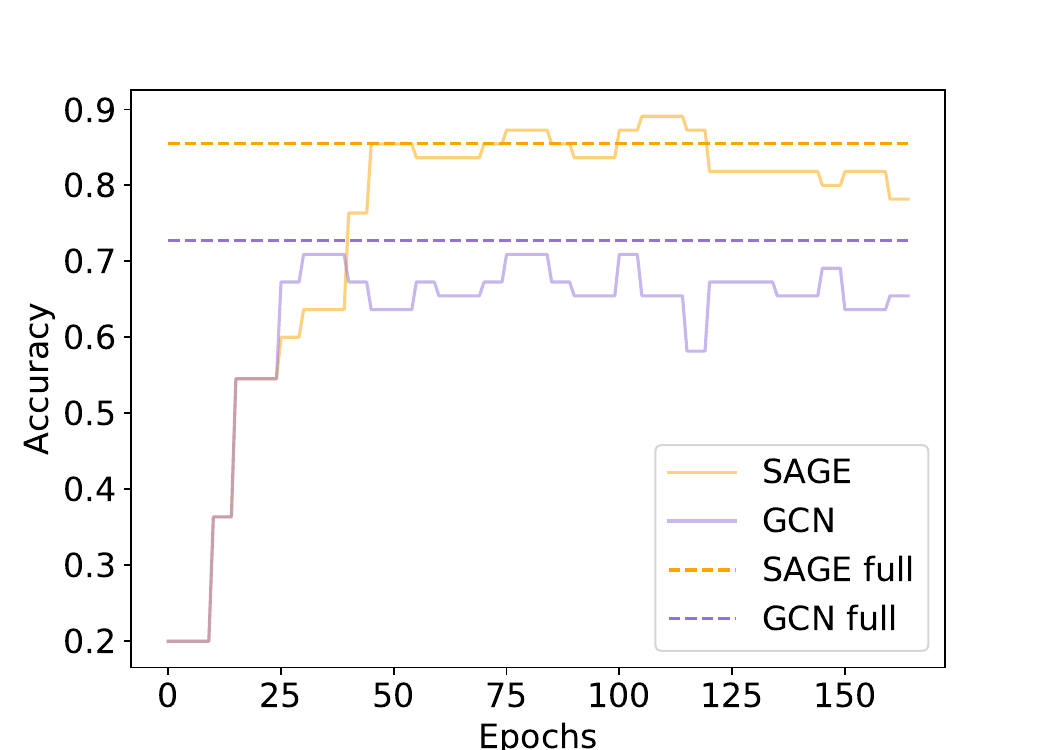}
\includegraphics[width=0.32\linewidth]{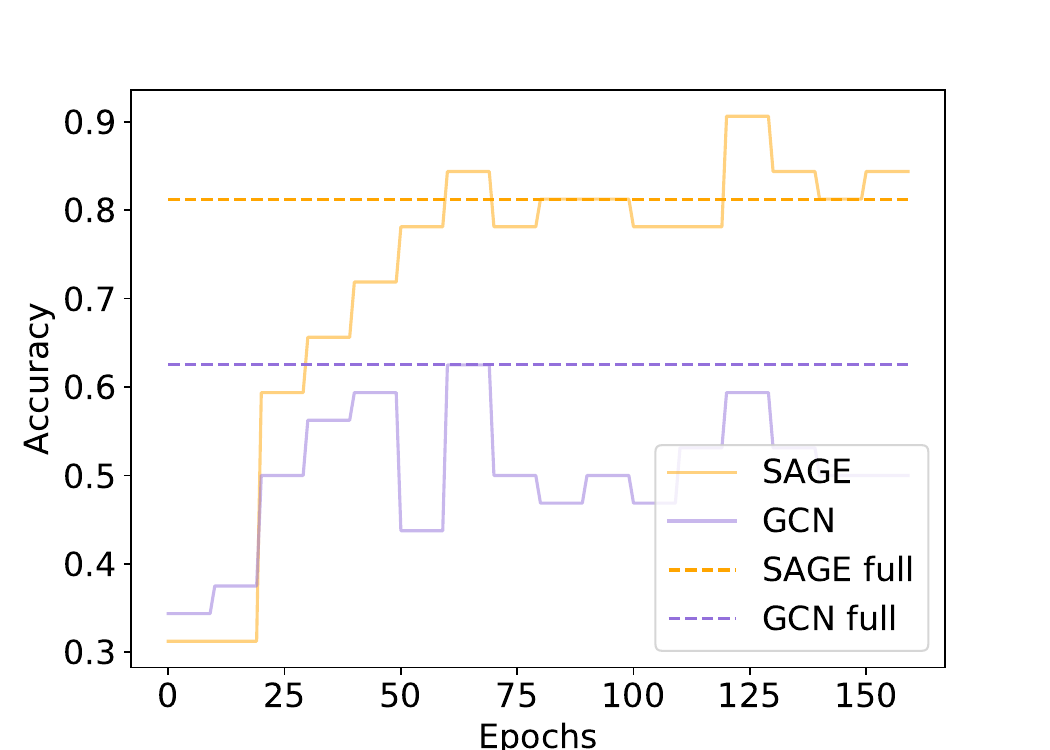}
\caption{Computational graph sampling results for PubMed with $K_1=K_2=32$ and batch size $32$. We consider three scenarios in terms of the small graph size $n$ and the graph sampling interval $\gamma$: $n=1500$, $\gamma=15$ epochs (left); $n=2000$, $\gamma=15$ epochs (center); $n=1500$, $\gamma=10$ epochs (right). Note that these graphs have size equal to approximately $10\%$ of the original graph size.}
\label{fig:comp_sampling}
\end{figure*}

\noindent \textbf{Computational graph sampling.}~Next, we consider GNNs which, in addition to employing node sampling, use the computational graph sampling strategy proposed by \citet{hamilton2017inductive}, called neighborhood sampling (see Appendix \ref{sec:app}). This technique consists of fixing parameters $K_\ell$ and, at each layer $\ell$, randomly sampling $K_\ell$ neighbors of each node from which to aggregate information. We fix $K_1=K_2=32$.   
The combinations of graph size $n$ and sampling interval $\gamma$ as well as the number of epochs and the learning rate are the same as in the previous experiment. The results are reported in Figure \ref{fig:comp_sampling}. We see that GNNs trained on collections of random subgraphs achieve comparable performance to GNNs trained on the full graph. Further, we observe some improvement in performance for GraphSAGE when the graph size is larger (center), and more variability in accuracy for both models when we increase the random subgraph sampling rate (right).
%and \red{outperform them in certain cases (e.g., GCN for Cora). We observe some overfitting for CiteSeer.}

Additional ablation results for two other citation networks---Cora and CiteSeer---can be found in Appendix \ref{appendix: saint}. We also provide additional results for GNNs without node or computational graph sampling in Appendix \ref{appendix: no_sampling}. 

\subsection{Application Example on a Very Large Graph}

Next, we consider a more realistic example on a very large graph: ogbn-mag. ogbn-mag is a heterogeneous citation network with 1,939,743 nodes representing authors, papers, institutions and fields of study \citep{hu2020open}. We focus exclusively on the paper-to-paper citation graph, which has 736,389 nodes. Due to memory limitations, we subsample it to 200,000 nodes. We consider a single experimental scenario using both node sampling with batch size $128$ and computational graph sampling with $K_1 = K_2 = 25$. The learning rate was $1e-2$.

The experiment results are reported in Figure \ref{fig:ogbn}. As before, the solid line corresponds to the GNN trained on a collection of randomly sampled subgraphs of size $n$, and the dashed line to the one trained on the full graph. 
We choose $n = 5000$ and resampling interval $\gamma = 10$ epochs for the former. Note that this choice of $n$ gives graphs with size equal to approximately 2.5\% of the size of full graph.

\begin{figure}
\centering
\includegraphics[width=0.6\textwidth]{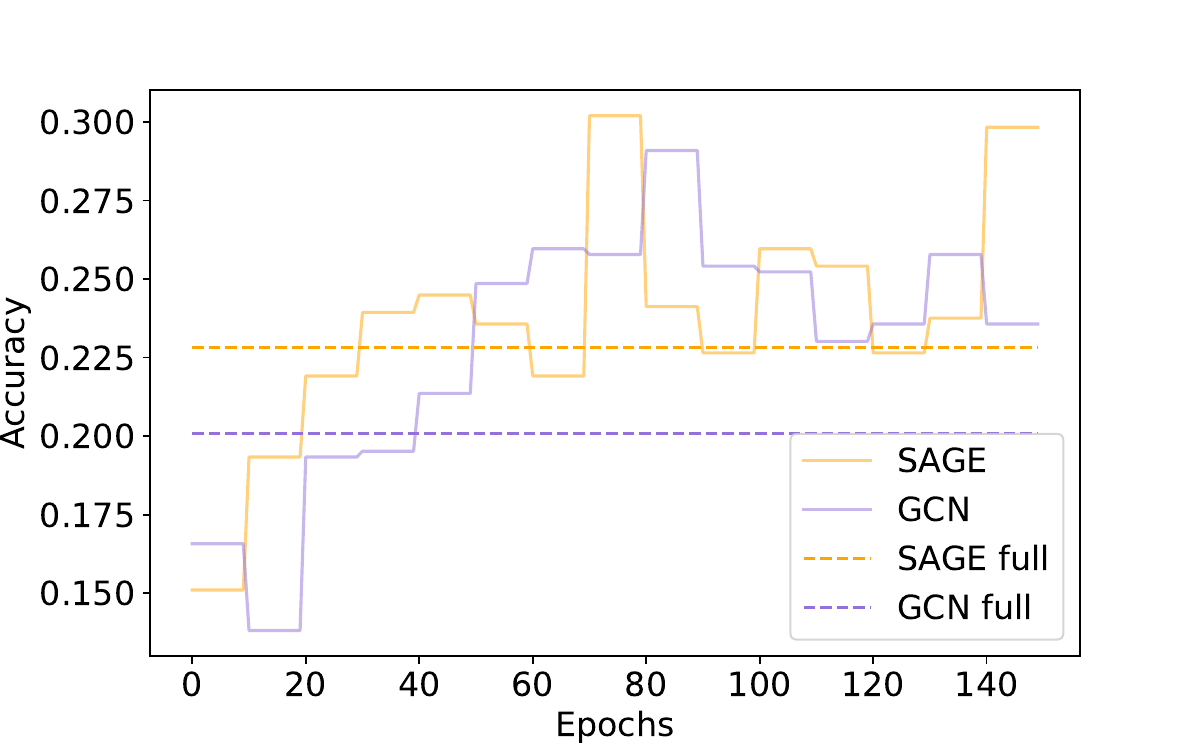}
\caption{Computational graph sampling results for ogbn-mag with $K_1 = K_2 = 25$ and batch size $128$. We consider one scenario in terms of the small graph size $n$ and the graph sampling interval $\gamma$: $n = 5000$, $\gamma = 10$ epochs.}
\label{fig:ogbn}
\end{figure}

\section{Conclusion and Future Steps}

 We have presented a novel theoretical framework for training GNNs on large graphs by leveraging the concept of local limits. Our algorithm guarantees convergence to an $\epsilon$-neighborhood of the learning GNN on the limit in $\mathcal{O}(1/\epsilon^2)$ training steps, which makes it a promising method for comparing different GNN architectures on large graphs efficiently.

Moving forward, a promising avenue for research is to explore the algorithm's robustness with adaptive sampling strategies, like those in ASGCN \citep{huang2018adaptive}, or with schemes that globally encompass the graph, exemplified by ClusterGCN \citep{chiang2019cluster}. It would be interesting to see how the convergence behavior of the algorithm is affected when the limit itself changes over time. Additionally, it would be valuable to explore whether the weights learned by sampling-based GNNs in the limit can be coupled with the actual GNN. Although it may not be possible in general, under certain assumptions, it may be feasible to establish such a connection, opening up new avenues for further research in this field.

% Question for future-self: What happens when the sampling scheme changes over time such as in ASGCN? ( Note that the limit itself could change over time in this case.)
% Can we couple sampling based GCN learned weights in the limit to the actual GCN? in general the answer should be no. Since the particular structure of graphs will change, for example assume the goal is to find the fraction of nodes with even degrees, then the answer in any sampling based GCN will be far from the actual answer. However, maybe under some assumptions we can prove it.

% % Bring the following in un-anonymous submission
% \section*{Acknowledgement}
%  Part of this work was done while the authors were visiting the Simons Institute for the Theory of Computing. We thank Megha Srivastava on her feedback on the initial draft of the paper.

\bibliography{myIEEEabrv,ref,bib-transferability}
\bibliographystyle{icml2023}

\pagebreak

\appendix
%%%%%%%%%%%%%%%%%%%%%%%%%%%%%%%%%%%%%%%%%%%%%%%%%%%%%%%%%%%%%%%%%%%%%%%%%%%%%%%
%%%%%%%%%%%%%%%%%%%%%%%%%%%%%%%%%%%%%%%%%%%%%%%%%%%%%%%%%%%%%%%%%%%%%%%%%%%%%%%

\section{More Related Work} \label{appendix: more_rel_work}

\noindent \textbf{Convergence and transferability of GNNs.}~The convergence of GNNs on sequences of graphs, also called transferability, has been studied in a number of works. \citet{ruiz20-transf} derive a non-asymptotic error bound for GNNs on dense graph sequences converging to graphons, and use it to show that GNNs can be trained on graphs of moderate size and transferred to larger graphs. These bounds are refined, and proved for a wider class of graphs, in \citep{ruiz2021transferability}. \citet{levie2019transferability} prove convergence and transferability of GNNs on graphs sampled from general topological spaces, and \citet{maskey2021transferability} particularize this analysis to graphons. \citet{keriven2020convergence} study the convergence and transferability of GNNs on graphs sampled from a random graph model where the edge probability scales with the graph size, allowing for graphs that are moderately sparse. \citet{roddenberry2022local} define local distributions of the neighborhoods of a graph to prove a series of convergence results for graph convolutional filters, and in particular their transferability across graphs of bounded degree. Note that this is different than what we do in this paper: we focus on \emph{sampling-based} GNNs, and prove convergence of their training on local subgraphs to GNNs trained on the large limit graph. More generally, random graph models have been a popular tool for understanding theoretical properties of GNNs, with several applications to topics including expressive power in community detection \citep{ruiz2022graph}, linear separability in semi-supervised classification \citep{baranwal2023effects}, heterophily \citep{ma2021homophily}.

\noindent\textbf{Pruning neural networks.}~ Another line of research closely related to our work revolves around neural network pruning during the training stage. This approach aims to improve training efficiency by selectively sampling connections between neurons \citep{zhu2017prune, gale2019state, strubell2019energy}. Notably, recent studies on the lottery ticket hypothesis \citep{frankle2018lottery, frankle2020linear} have demonstrated that by sampling subnetworks during training, comparable accuracy to the full network can be achieved while significantly reducing its complexity. In contrast, our work takes a divergent trajectory as we shift our focus towards sampling data characterized by an intrinsic network structure, instead of manipulating the neural network connections.

\section{GCN and GraphSAGE} \label{appendix: gcn_sage}

For illustrative purposes, we focus on the GCN and GraphSAGE architectures. We also experiment with these architectures in Section \ref{sec: experiment}.

\noindent \textbf{GCN.}~The layerwise propagation rule of GCN is given by \citep{kipf17-classifgcnn}
\begin{align} \label{eqn:gcn}
\mathbf H^{(\ell+1)} = \sigma\left(\mathbf D^{-1/2} \mathbf A \mathbf D^{-1/2} \mathbf H^{(\ell)} \mathbf W^{(\ell+1)}\right)
\qquad \text{ and }\qquad \mathbf H^{(0)} =  \mathbf X,
\end{align}
where $\mathbf D = \text{diag}(\mathbf A \boldsymbol{1})$ is the degree matrix; $\mathbf W^{(\ell+1)} \in \mathbb R^{F_\ell \times F_{\ell+1}}$ are the learnable convolution weights at layer $\ell+1$; and $\sigma$ is a pointwise nonlinearity such as the ReLU or the sigmoid. The adjacency matrix $\mathbf A$ is modified to include self-loops, hence the \texttt{AGGREGATE} and \texttt{COMBINE} operations can be written as a single operation.~Variants of GCN may consider $K$-hop graph convolutions instead of one-hop \citep{defferrard17-cnngraphs,gama18-gnnarchit,du2018graph}.

%\todo{\bf amin: is 2020 the first reference to multi-hop GCNs?}

\noindent \textbf{GraphSAGE.}~The layerwise propagation rule of GraphSAGE is given by \citep{hamilton2017inductive}
\begin{align} \label{eqn:sage}
\begin{split}
\mathbf{h}^{(\ell+1)}_{N(v)}&=\texttt{AGGREGATE}_\ell\left(\{\mathbf h^{(\ell)}_u,u\in N(v)\}\right) \\
\mathbf h^{(\ell+1)}_v& =  \sigma\left(\mathbf W^{(\ell+1)} \cdot \texttt{CONCAT}\left(\mathbf{h}^{(\ell)}_{v},\mathbf{h}^{(\ell+1)}_{N(v)}\right)\right)\qquad \text{ and }\qquad \mathbf h^{(0)}_v = \mathbf x_v,
\end{split}
\end{align}
where $\mathbf W^{(\ell+1)} \in \mathbb R^{F_{\ell+1}\times 2F_\ell}$ are learnable weights and $\sigma$ is a pointwise nonlinearity. Typical \texttt{AGGREGATE} operations are the mean, the sum, and the max. 

%In a Graph Convolutional Network, the computation process can be viewed as a message-passing algorithm. At each step, every node in the graph sends messages to its neighboring nodes. These messages are then combined to update the state of each node. For a given node $v$, the subgraph $G_v$ with $L$ layers on which the message-passing takes place is called the computational graph. The computational graph can be deterministic as in full-GCN \cite{kipf17-classifgcnn} or randomly sampled from the local neighborhood of a node which results in different variants of GCN and we will discuss later in Section~\ref{sec: sampling GCN}. 

\section{Classical Sampling-Based GNN Training Algorithm}\label{appendix: alg}
\begin{center}
\begin{minipage}{0.5\textwidth}
\begin{algorithm}[H]
    \footnotesize
   \caption{Classical Training}
   \label{alg: classic}
%\begin{algorithmic}
   %{\bfseries Input:} Graph $G(V,E)$; gradient sampler $\nu_g$; computational graph sampler $\nu_C$; aggregator functions $\text{AGGREGATE}_\ell$ for $\ell\in\{0,\dots,L\}$; depth $L$.

   \While{$|\nabla\mathcal L|>\epsilon$}{
   $\mathbf Z=$\texttt{SamplingBasedGNN}($G$)
   %;$\nu_g$;  $\nu_C$; $\texttt{AGGREGATE}_\ell, \texttt{COMBINE}_\ell$ for $\ell\in\{0,\dots,L\}$; $L$).
   
   $\nabla_{\mathbf W_t}\tilde{\mathcal L}_{\nu_g}(\mathbf Z)=\frac{1}{|V_B|}\sum_{v\in V_B}\frac{1}{\nu_g(v)} \nabla \mathcal L(\mathbf z_v)$
   
   $\mathbf{W}_{t+1}=\mathbf{W}_t-\eta \nabla_{\mathbf W_t} \nabla\tilde{\mathcal L}_{\nu_g}$
   
   $t\leftarrow t+1$
   }
%\end{algorithmic}
\end{algorithm}
\end{minipage}
\end{center}

\section{Examples of Almost Local Functions}\label{appendix: GL proof}

\begin{proposition}\label{prop: almost local fnc}
For a sequence of convergent graphs $\{G_n\}_{n\in\mathbb N}$, the following holds. 
\setlist{nolistsep}
\begin{enumerate}[noitemsep]
    \item Negative sampling:~ Let $\sigma$ be the Sigmoid function and $x\cdot y$ be the dot-product of two vectors $x$ and $y$, and let $g_\ell: \mathcal{G}_*:\mathbb R^K$ be a bounded and continuous function on rooted graphs with a finite radius $\ell\geq 0$. Then the functions $\smash{f(G_n,v,u)=\sigma \big(g_\ell(G,v)\cdot g_\ell(G,u)\big)}$ 
    and $\tilde f(G_n,v)=\mathbb{E}_{u\sim\mathcal{D}_n}[f(G_n,v,u)]$ are almost local. Here $D_n$ is a  distribution on nodes of $G_n$ such that its density, i.e., $n\mathbb P_{v\sim D_n}$, is local.
    % , and
    % ~Consider 
    %any bounded and continuous function $h$,  any almost local function $g$, and
    % a sequence of distribution $\{D_n\}_{n\in\mathbb N}$ on nodes $D_n$ with the limit $D$,
    % an almost local sampling weight function $D_n$,
    % such that 
    % the distance of any two nodes $u$ and $v$ drawn independently from $D_n$ grows with the size of the graph, i.e, $\smash{dist_{G_n}(u,v)\overset{\mathbb P}{\to}\infty}$.
    % Then the functions $\smash{f(G_n,v,u)=\sigma \big(x_v\dot y_v\big)}$ 
    % and $\tilde f(G_n,v)=\mathbb{E}_{u\sim\mathcal{D}_n}[f(G_n,v,u)]$ are almost local as well.
    % Further,  $D_n$ is a uniform random distribution 
    % and $h(.)=\sigma(\text{dot}(.))$ satisfies the requirements. 
    
    \item Normalized adjacency matrix: ~When the degree sequence is uniformly integrable\footnote{A random variable $X$ is said to be uniformly integrable if the probability that it exceeds a certain threshold value approaches zero as the threshold increases towards infinity, i.e., for any $\epsilon>0$ there exists $K_\epsilon$ such that $\mathbb P(X>K_\epsilon)<\epsilon$.}, then re-weighting any local function with respect to the normalized adjacency matrix is almost local. Formally, if $g$ is a bounded and continuous function, then   the function
$\smash{f(G_n,v)=g(G_n,v)\frac{deg(v)}{\frac{1}{|V(G)|}\sum_u deg(u)}}$ is almost local, and the limit can be written as,
\[\mathbb{E}_{v\sim\mathcal{P}_n}[f(G_n,v)|G_n]\overset{\mathbb{P}}{\to}\mathbb{E}_{(G,o)\sim\mu}\Big[g(G,o)\frac{deg(o)}{\bar{d}}\Big].\]
\end{enumerate}
\end{proposition}
 % Note that almost-locality is preserved under addition, subtraction, multiplication, and finite composition as long as the output of the function is bounded. 
 % In the first part of the proposition, we prove almost the locality of a  generalization of  negative sampling. To apply our result to GraphSAGE, which uses  negative sampling in unsupervised learning \cite{hamilton2017inductive}, we can set $g(G,v\mathbf x)$, and $h(.)=\sigma(\text{dot}(.))$ and then apply the proposition.

\begin{proof}[Proof of Proposition~\ref{prop: almost local fnc}]
% the intuition
{\bf Part 1 (Negative sampling).}
We start with the proof for negative sampling. % Assume that node-embeddings $\mathbf{z_n}$ depend on $\ell$-neighborhoods of the nodes, i.e., $\ell$ is the number of layers in GNN.
The main idea is that, the value of $g_\ell(G,v)$ only depends on $\ell$ neighborhood of the node $v$. On the other hand, given the assumption on $D_n$ two random nodes drawn independently from $D_n$ are with high probability more than $2\ell$ far apart. As a result, their corresponding  node embedding $z_v=g_\ell (G,v)$ and $z_u=g_\ell(G,u)$ is with high probability independent of each other.  So, we can use the fact that if two independent random variables converge in probability, then
their product also converges. 

Next, we formalize this intuition using the second-moment method. The convergence of the first moment is trivial by the 
\begin{align*}
    \mathbb E_{G_n}\big[\mathbb E_{u,v\sim\mathcal{P}_n}[f(G_n,u,v)|G_n]\big]=\mathbb E_{G_n}\big[\mathbb E_{u,v\sim\mathcal{P}_n}[\sigma(g_v\cdot g_u)|G_n]\big]
\end{align*}
I want to prove it converges to
\[ \mathbb E_{(G^{(1)},o_1),(G^{(2)},o_2)\sim \mu}\big[\sigma (g(G^{(1)},o_1), g(G^{(2)},o_2))\big]
\]
% \todo{needs to be fixed}
% The proof finishes by applying Corollary 2.20 of \citep{RemcoVol2}, which shows that the distance of two vertices $u,v\in V(G_n)$ chosen uniformly at random grows with $n$.
% formalization with second moment method

{\bf Part 2 (Normalized adjacency matrix).} The proof follows the classic proofs on the locality of functions such as centrality coefficients as in \citep[Chapter 2.4]{RemcoVol2}.
First, note that we can rewrite:
\begin{align}\label{eq:apndx - normalized}
    \mathbb E_{v\sim\mathcal{P}_n}[f(G_n,v)|G_n]= \frac{\mathbb E_{v\sim\mathcal{P}_n}[deg(v)g(G_n,v)|G_n]}{\mathbb E_{v\sim\mathcal{P}_n}[deg(v)|G_n]}.
\end{align}
The local convergence in probability \eqref{eq: lwc function definition} applies to only bounded and continuous functions. 
However, the graph might have unbounded degrees that could make the enumerator or the denominator of \eqref{eq:apndx - normalized} unbounded. We will control its effect using the uniform integrability assumption on the degree sequence.

First, we will focus on the enumerator.
In particular, for any fixed integer $\Delta>0$, we can split 
\begin{align*}
    \mathbb E_{v\sim\mathcal{P}_n}[deg(v)g(G_n,v)|G_n]= \mathbb E_{v\sim\mathcal{P}_n}[ deg(v)g(G_n,v) 1(deg(v)\leq \Delta)|G_n]\\ \qquad+ \mathbb E_{v\sim\mathcal{P}_n}[g(G_n,v) deg(v) 1(deg(v)> \Delta)|G_n].
\end{align*}
The first term, $deg(v)g(G_n,v) 1(deg(v)\leq \Delta)$, is already a bounded continuous function and we can apply \eqref{eq: lwc function definition}, to get 
\begin{align*}
   \mathbb E_{v\sim\mathcal{P}_n}[deg(v)g(G_n,v) 1(deg(v)\leq \Delta)|G_n]\overset{\mathbb P}{\to}   \mathbb E_{(G,o)\sim\mu}[ deg(o) g(G,o) 1(deg(o)\leq \Delta)].
\end{align*}
For the second term, we bound it using uniform integrability. Since $g$ is a bounded function, there exists $M$ as an upper bound for it. The uniform integrability implies that for any $\epsilon>0$ there exists a large enough $N_\epsilon$, such that for all $n>N_\epsilon$,
\begin{align}
 \mathbb E_{v\sim\mathcal{P}_n}[deg(v)g(G_n,v)]\leq  \mathbb E_{v\sim\mathcal{P}_n}[deg(v)M]\leq \epsilon^2.
\end{align}
Then using Markov inequality, 
\begin{align}
 \mathbb P\Big(\mathbb E_{v\sim\mathcal{P}_n}[deg(v)g(G_n,v)]\geq \epsilon|G_n\Big)\leq  \frac{1}{\epsilon}\mathbb E_{v\sim\mathcal{P}_n}[deg(v)M]\leq \epsilon.
\end{align}
Therefore, 
\begin{align*}
   \mathbb E_{v\sim\mathcal{P}_n}[deg(v)g(G_n,v)|G_n]\overset{\mathbb P}{\to}   \mathbb E_{(G,o)\sim\mu}[ deg(o) g(G,o) ].
\end{align*}
Similarly, we can get 
\[
\mathbb E_{v\sim\mathcal{P}_n}[deg(v)|G_n]\overset{\mathbb P}{\to}   \mathbb E_{(G,o)\sim\mu}[ deg(o)  ],
\]
which would prove the second part of the proposition.
\end{proof}

%%%%%%%%%%%%%%%%%%%%%%%%%%%%%%%%%%%%%%%%%%%%%%%%%%%%%%%%%%%%%%%%%%%%%%%%%%%%%%%
%%%%%%%%%%%%%%%%%%%%%%%%%%%%%%%%%%%%%%%%%%%%%%%%%%%%%%%%%%%%%%%%%%%%%%%%%%%%%%%

\section{Proof of Theorem~\ref{thm: stochastic GNN}}\label{apendix: proof main}
% In this section, we prove Theorem~\ref{thm: stochastic GNN} and bring the proof of the second theorem based on importance sampling in Appendix~\ref{appendix: importance sampling}. 
We begin the proof by relating the gradient of the loss function on finite samples to the gradient of the loss function in the limit.
To formalize the proof, let $\mathcal L_{\nu_g}(W_t,G)$  be the loss of applying $W_t$ (coefficients at iteration $t$ of the algorithm) to graph $G$ with node sampling $\nu_g$ \footnote{ For simplicity, we don't show the input feature in this representation. But all results apply to convergent graphs with input features.}.

In this section, we present the proof of Theorem~\ref{thm: stochastic GNN}, which provides a theoretical guarantee for the convergence of the sampling-based GNN training algorithm presented in Algorithm~\ref{alg: training}. 
%We also refer interested readers to Appendix~\ref{thm: stochastic GNN importance sampling} for the proof of the second theorem based on importance sampling.

To start the proof, we first establish a connection between the gradient of the loss function on finite samples and the gradient of the loss function in the limit. Specifically, we use the notation $\mathcal L_{\nu_g}(W_t,G)$ to denote the loss of applying the coefficients $W_t$ at iteration $t$ to a graph $G$ with node sampling $\nu_g$\footnote{Note that for simplicity, we do not explicitly show the input feature in this notation, but all results apply to convergent graphs with input features.}.

Our proof proceeds in two main steps. First, we show that in each iteration of the algorithm, the gradient of the loss on the sampled graphs is close to the expected gradient of the loss in the limit infinite graph (Lemma~\ref{lm: nabla bound-stochastic}). Second, we show that the loss on finite graphs decreases at each step with a high probability for the correct choice of learning rates (Lemma~\ref{lm: decay in loss}).

\begin{lemma}\label{lm: nabla bound-stochastic}
Given the assumptions of Theorem~\ref{thm: stochastic GNN},  for any  $t>0$, then the loss in iteration $t$ of Algorithm~\ref{alg: training} on a finite graph $G_n$with appropriately sized  mini-batches is with high probability close to the expected loss on the limit. In particular, for any $\epsilon>0$ there exists $N_\epsilon$ and $\mathcal{Y}_\epsilon$, such that for $n>N_\epsilon$, when either the mini-batch is large enough
 (i.e., $V_B\geq\mathcal{Y}_\epsilon$) but at the same time not too large (i.e., $V_B=o(\sqrt{|V(G_n)|})$)  or  includes all nodes (i.e., $|V_B|=|V(G_n)|$), 
\[\mathbb P\Big(|\nabla\tilde{\mathcal L}_{\nu_g}(W_t,G_n)- \mathbb E_{\mu}[\nabla \mathcal L (W_t,G)]|\geq\epsilon\Big)\leq \epsilon,\]
where the probability is over the possible randomness of $G_n$, the randomness of SGD, and the computational graph sampler $\nu_C$.
\end{lemma}
\begin{remark}
The Lemma requires an upper bound on the minibatch size to ensure that the embeddings and their loss are independent. This is necessary to control for the possibility of the local neighborhood of sampled nodes not being disjoint. 
\end{remark}
\begin{proof}
First, consider the case that the minibatch $V_B$ includes all nodes, i.e., we use gradient descent instead of SGD to compute the gradient of loss. Then, the loss on the finite graph converges to the limit by the assumption that the loss is almost local. In particular, for any $\epsilon'>0$ there exist $N_{\epsilon'}$  such that for $n>N_{\epsilon'}$,
\begin{equation}\label{eqn: nabla GD convg}
    \mathbb P\Big(|\nabla\tilde{\mathcal L}(W_t,G_n)- \mathbb E_{\mu}[\nabla \mathcal L (W_t,G)]|\geq\epsilon'\Big)\leq \epsilon',
\end{equation}
where the probability is over the possible randomness of $G_n$, and the weights $W_t$ (which in turn depends on the randomness of the computational graph sampler $\nu_C$). In fact, as the size of the graphs in the sequence ${G_n}_{n\in\mathbb N}$ grows, we have convergence in probability,
\[\nabla\tilde{\mathcal L}(W_t,G_n)\overset{\mathbb P}{\to}\mathbb E_{\mu}[\nabla \mathcal L (W_t,G)].\] Note that there is no expectation on the left-hand side of neither Lemma \ref{lm: nabla bound-stochastic} nor \eqref{eqn: nabla GD convg} nor the above expression, as the loss (and its gradient) are invariant to the choice of root in $G_n$.
%\todo{explain why this is expectation over all nodes on the left hand side}

Next, we consider the case where $V_B$ is an i.i.d subset of nodes. The idea is to use of the fact that the loss is almost local, and hence only depends on a bounded neighborhood of  a few nodes. So, if we draw a small enough minibatch $|V_B|= o(\sqrt{|V(G_n)|})$ then the loss of different nodes should be independent of each other, and we can use Hoeffding bound to control the error. We formalize this idea next.

Let $v_1,\ldots, v_{V_B}$ be the sampled nodes in the minibatch.  First, note that since the loss is almost local, there exists $K_\epsilon$ such that there exists  a function on $K_\epsilon$ neighborhood of nodes which is a $(1-\epsilon/4)$ approximation of the loss.
We continue by proving that the local neighborhoods of the sampled nodes in the minibatch are disjoint with high probability. For this purpose, let $I_{V_B}=\{\forall i,j\in[V_B], \quad B_K(G_n,v_i)\cap B_K(G_n,v_i)=\emptyset\}$ be the event that $K$-neighborhoods are disjoint. We want to prove that there exists $N_\epsilon$ such that for $|V(G_n)|\geq N_\epsilon$ and any $|V_B|=o\big( \sqrt{ |V(G_n)|}\big)$,
\begin{equation}\label{eq: Y-B}
    \mathbb P(I_{V_B}\text{ does not happen.})\leq \frac{\epsilon}{4}.
\end{equation}
Define ${V}_{K,\Delta}$ as the set of nodes such that the maximum degree in their $K$ neighborhood is at most $\Delta$.
Since the sequence $\{G_n\}_{n\in\mathbb N}$ converges in the local sense,  for all
$\epsilon>0$, there exists $\Delta<\infty$ and $N_\epsilon'<\infty$ such that for $n\geq N_\epsilon'$, with probability $1-\frac{\epsilon}{4}$ 
we have $\frac{|V_{k_\epsilon,\Delta}|}{n}\geq 1-\frac{\epsilon}{4}$. Let $E_\epsilon$ be the event that
$\frac{|V_{k_\epsilon,\Delta}|}{n}\geq 1-\frac{\epsilon}{4}$.   Then by a union bound
\[\mathbb P (I_{V_B} \text{ does not happen})\leq  \mathbb P(I_{V_B}\text{ does not happen}\mid E_\epsilon)+\mathbb P(E_\epsilon)\leq |V_B|^2\frac{\Delta^K}{|V(G_n)|}+\frac{\epsilon}{8}.\]
Since $\Delta$ and $K$ are independent from $|V(G_n)|$, by increasing $N_\epsilon'$, if necessary, we can assume that $|V_B|\Delta^K\leq  \frac{\epsilon}{8}|V(G_n)|$ and hence proving \eqref{eq: Y-B}.

Now, conditioned on the event $I_{V_B}$, we can apply Hoefdding bound. Let $X_i=\mathcal L(G_n,v_i)$, where $v_i\sim\mathcal{P}_n$ is drawn uniformly at random. 
Then,
\begin{align*}
    \mathbb P(|\nabla\tilde{\mathcal{L}_{V_B}}-\nabla\tilde{\mathcal{L}_{V_B}}|\geq \frac{\epsilon}{2})
    &\leq\mathbb P(|\nabla\tilde{\mathcal{L}_{V_B}}-\nabla\tilde{\mathcal{L}_{V_B}}|\geq\frac{ \epsilon}{2}\mid  I_{V_B}) +\mathbb P (I_{V_B} \text{ does not happen}) \\ 
    &\leq exp(-{\frac{|V_B|\epsilon}{C^2}})+\frac{\epsilon}{4}.
\end{align*}
So, if we choose $|V_B|$ between $o( \sqrt{|V(G_n)|})$ and $\Omega(\log(1/\epsilon))$, we will get
\[ \mathbb P(|\nabla\tilde{\mathcal{L}_{V_B}}-\nabla\tilde{\mathcal{L}_{V_B}}|\geq \frac{\epsilon}{2})\leq \frac{\epsilon}{2}.\]
This together with \eqref{eqn: nabla GD convg} gives the result.
\end{proof}

In the following lemma, we demonstrate that the loss in the limit decays in each iteration of the algorithm by applying the above lemma. Notably, we only rely on the Lipschitz property of the loss function in our proof and do not make any other use of the locality of  the loss function, as long as Lemma~\ref{lm: nabla bound-stochastic} holds.
\begin{lemma}\label{lm: decay in loss}
Given the assumption of Theorem~\ref{thm: stochastic GNN}, fix some $\epsilon>0$. Also assume the learning rate $\eta$ is smaller than $1/C$, where $C$ is the loss function's Lipschitz constant.
Then there exists some $N_\epsilon$ such that if  Algorithm~\ref{alg: training} is trained on graphs of size larger than $N_\epsilon$, then
the loss on the limiting graph decreases in  each iteration,
\begin{align*}
\mathbb P\Big(\mathbb E_\mu[\mathcal L(W_{t+1},G)]\leq\mathbb E_\mu[\mathcal L (W_t,G)]-\frac{\eta}{4}||\mathbb E_\mu[\nabla\mathcal L(W_t,G)]||^2\Big)\geq 1-\epsilon,
\end{align*}
where the probability is over the possible randomness of $G_n$, the randomness of SGD, and the computational graph sampler $\nu_C$.
\end{lemma}

\begin{proof}
We use a helper function to smooth out the difference  between the two steps of the loss function. 
Let (in the following the expectation is over the limit $G$),
\[g(\epsilon)=\mathbb E_\mu[\mathcal L(W_{t}-\epsilon\eta \nabla\mathcal L(W_t, G_n), G))].\]
Then $g(1)=\mathbb E[\mathcal L(W_{t+1},G)]$ and $g(0)=\mathbb E[\mathcal L(W_{t},G)]$. This definition of a helper function has been classically used in the literature to prove convergence of the loss \citep{cervino2021training,bertsekas2000gradient}.

By differentiating the helper function, $\frac{\partial g}{\partial \epsilon}=-\eta \nabla\mathcal L(W_{t},G_n) \mathbb E_\mu[\nabla\mathcal L(W_{t}-\epsilon\eta \mathcal L(W_{t},G_n), G))].$
So, we can write, 
\begin{align*}
    g(1)-g(0)&=\int_0^1\frac{\partial g}{\partial \epsilon }d\epsilon.\\
    &=-\eta \nabla\mathcal L(W_{t},G_n)\int_0^1\mathbb E_\mu[\nabla\mathcal L(W_{t}-\epsilon\eta \mathcal L(W_{t},G_n), G))]d\epsilon
\end{align*}
Then we add and subtract $\mathbb{E}_\mu[\nabla\mathcal L(W_{t},G)]$ to get
\begin{align*}
  & \mathbb E_\mu[\mathcal L(W_{t+1},G)]-\mathbb E_\mu[\mathcal L(W_{t},G)]\\
  &  =-\eta \nabla\mathcal L(W_{t},G_n)\mathbb E[\nabla\mathcal L(W_{t},G)]  +\eta\nabla\mathcal L(W_{t},G_n)\Big( \int_0^1\nabla \mathcal L(W_{t}
  -\epsilon\eta \nabla\mathcal L(W_t,G_n), G)\\&-\mathbb E_\mu[\nabla\mathcal L(W_t, G)]d\epsilon\Big)
\end{align*}
Now since $\nabla \mathcal L$ is Lipschitz (with constant $C$),
% Then by applying Lemma~\ref{lm: nabla bound},
\begin{align*}
  &  \mathbb E_\mu[\mathcal L(W_{t+1},G)]-\mathbb E_\mu[\mathcal L(W_{t},G)]\\
  &  \leq -\eta \nabla\mathcal L(W_{t},G_n)\mathbb E[\nabla\mathcal L(W_{t},G)] +\eta| \nabla\mathcal L(W_{t},G_n)|\Big( \int_0^1\epsilon\eta C||\nabla\mathcal L(W_t, G_n)||d\epsilon\Big)\\
  &=-\eta \nabla\mathcal L(W_{t},G_n)\mathbb E[\nabla\mathcal L(W_{t},G)]+\frac{\eta^2 C}{2}||\nabla\mathcal L(W_t, G_n)||^2\\
  &=\frac{\eta^2 C-\eta}{2}||\nabla\mathcal L(W_t, G_n)||^2-\frac{\eta}{2}\Big(||\mathbb E_\mu[\nabla\mathcal L(W_{t},G)]||^2-||\nabla\mathcal L(W_t, G_n)-\mathbb E_\mu[\nabla\mathcal L(W_{t},G)]||^2\Big).
\end{align*}
To finish the proof, we can bound the first term by choosing the learning rate $\eta$ smaller than $1/C$, and the second term by using Lemma~\ref{lm: decay in loss}.
\end{proof}

Now, we are ready to prove the theorem.
  For proof, we analyze the stopping time $t^*$, which is the first iteration at which the expected gradient of the loss with respect to the coefficients falls below a threshold $\epsilon$. Then we use Lemma~\ref{lm: decay in loss} to bound this stopping time.
\begin{proof}[Proof of  Theorem~\ref{thm: stochastic GNN}]
Given  $\epsilon>0$ define the stopping time 
\[t^*=\inf_t\{\mathbb{E}_{\mu,W_t}(\nabla \mathcal L(W_t,G))\leq \epsilon\},\]
where the expectation is both over the randomness of the limit $G$ and the coefficients $W_t$. Note that the randomness of $W_t$ is due to the randomness of sampling the graph $G_n$. %\red{(L: Why is the expectation also over $W_t$? Is it because of the randomness of $W_0$?).}\blue{Y: no it's because of randomness of $G_n$. }
We write loss as the sum of differences of loss in each iteration,
\begin{align*}
    \mathbb{E}_{\mu,W_t}\big( \mathcal L(W_{0},G)-\mathcal L(W_{t^*},G)\big)=
        \mathbb{E}_{\mu,W_t}\big( \sum_{t=0}^{t^*-1}\mathcal L(W_{t},G)-\mathcal L(W_{t+1},G)\big).
\end{align*}
We can take the expected value with respect to the randomness of $t^*$,
\begin{align*}
   \mathbb E_{t^*} \mathbb{E}_{\mu,W_t}\big(\mathcal L(W_{0},G)-\mathcal L(W_{t^*},G)\big)= \sum_{t^*=0}^{\infty}  \mathbb{E}_{\mu,W_t}\big(     \sum_{t=0}^{t^*-1} \mathcal L(W_{t},G)-\mathcal L(W_{t+1},G)\big)\mathbb P(t^*).
\end{align*}
    By applying Lemma~\ref{lm: decay in loss} for $t<t^*$,
    \[\mathbb P\Big( \mathbb{E}_\mu\big(\mathcal L(W_{t},G)-\mathcal L(W_{t+1},G)\big)\geq \frac{\eta}{4}\epsilon^2\Big)\geq 1-\epsilon.\]
    By applying this to the previous inequalities:
    \begin{align*}
   \mathbb E_{t^*} \mathbb{E}_{\mu,W_t}\big( \mathcal L(W_{0},G)-\mathcal L(W_{t^*},G)\big)\geq
     \frac{\eta\epsilon^2(1-\epsilon)}{4} \sum_{t^*=0}^{\infty}  t^*\mathbb P(t^*)=     \frac{\eta\epsilon^2(1-\epsilon)}{4}\mathbb E[t^*].
\end{align*}
Since loss is non-negative
\[ \frac{4}{\eta\epsilon^2(1-\epsilon)}\mathbb E_{t^*} \mathbb{E}_{\mu}\big( \mathcal L(W_{0},G)\big)\geq  \mathbb E[t^*] .\]
\end{proof}

% As an application of the proof of above lemma, we state the gradient convergence for the case that $\nu_g$ is proportional to the normalized adjacency matrix.

% \begin{corollary}
% Given the assumptions of Theorem~\ref{thm: GNN} and local finite samplers $\nu_C$ (Assumptions~\ref{assum: local} and \ref{assum: finite sampler}), further assume the node sampling is based on normalized adjacency matrix. Then if the second moment of (weighted) degrees is uniformly integrable, the result of Lemma~\ref{lm: nabla bound-stochastic} hold.
% \end{corollary}

% \begin{theorem}\label{thm: stochastic}
% Let ${G_n}_{n\in\mathbb N}$ be a sequence of convergent graphs (Assumption~\ref{assum: GL}). Then, there exists a small enough learning rate $\eta$ such that for any $\epsilon>0$, training sampling based GNN as in Algorithm~\ref{alg: training} with a Lipschitz loss function (Assumption~\ref{assum: lipschitz}) on the sequence of convergent graphs with learning rate $\eta$ converges to the $\epsilon$-neighborhood of learning the sampling-based GNN on the limit (with the same gradient and computational graph sampling functions) in $\tilde O(\frac{1}{\epsilon^2})$ expected steps.
% \end{theorem}

%%%%%%%%%%%%%%%%%%%%%%%%%%%%%%%%%%%%%%%%%%%%%%%%%%%%%%%%%%%%%%%%%%%%%%%%%%%%%%%
%%%%%%%%%%%%%%%%%%%%%%%%%%%%%%%%%%%%%%%%%%%%%%%%%%%%%%%%%%%%%%%%%%%%%%%%%%%%%%%

\input{previous-gnn-models}

\section{Experiment Details} \label{appendix: experiment_details}

\noindent \textbf{Citation networks.}~Cora, CiteSeer, PubMed and ogbn-mag are citation networks commonly used to benchmark GNNs. Their nodes represent papers, and their edges are citations between papers (in either direction). Each paper is associated with a bag-of-words vector, and the task is to use both the citation graph and the node information to classify the papers into $C$ classes. Relevant dataset statistics are presented in Table \ref{tab:stats}.

\noindent \textbf{Training details.}~In all experiments, we use PyTorch Geometric \citep{fey2019fast} and consider a GraphSAGE \citep{hamilton2017inductive} and a GCN \citep{kipf17-classifgcnn} architectures with $2$ layers and embedding dimensions $64$ and $32$ respectively ($64$ and $64$ for ogbn-mag). In the first and the second layers, the nonlinearity is a ReLU, and in the readout layer, a softmax. We minimize the negative log-likelihood (NLL) loss and report the classification accuracy. We consider the train-validation-test splits from the full Planetoid distribution from PyTorch Geometric for PubMed, Cora, and CiteSeer; and from Open Graph Benchmark for ogbn-mag. To train the models, we use ADAM with the standard forgetting factors \citep{kingma17-adam}.

\section{Additional Experiments on Cora and CiteSeer} \label{appendix: saint}

In this section, we repeat the experiments of Section \ref{sec: experiment} for the Cora and CiteSeer datasets, whose corresponding dataset statistics can be found in Table \ref{tab:stats}. The results for node sampling, computational graph sampling, and no sampling are described below.

\begin{figure*}[t]
\centering
\includegraphics[width=0.32\linewidth]{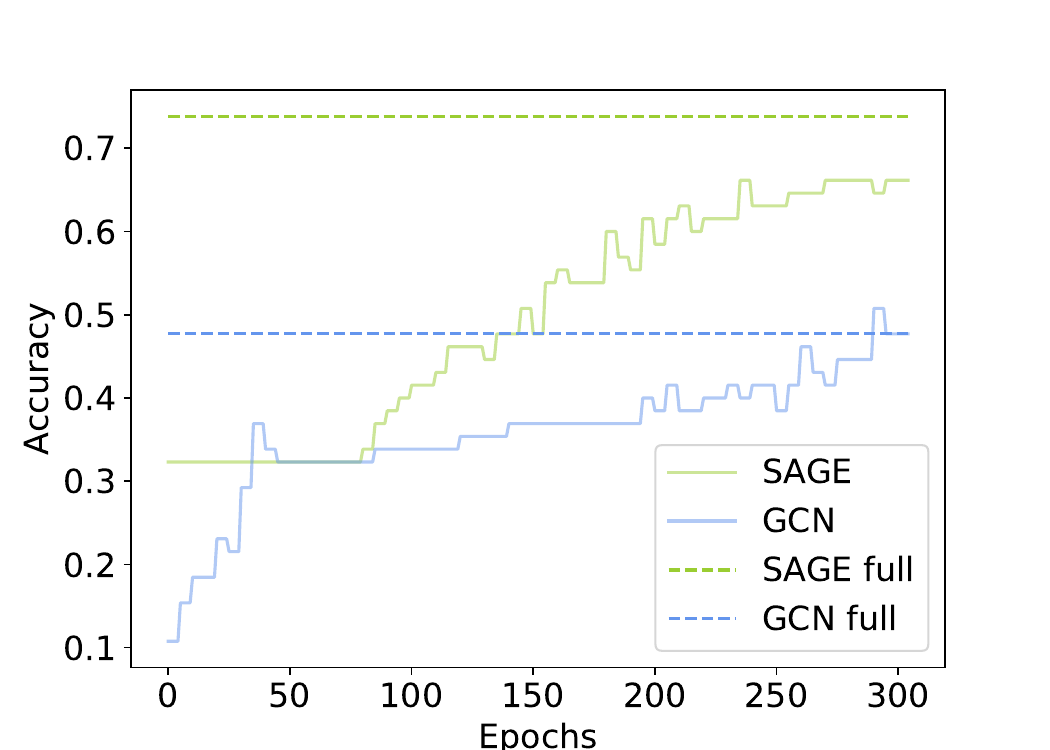}
\includegraphics[width=0.32\linewidth]{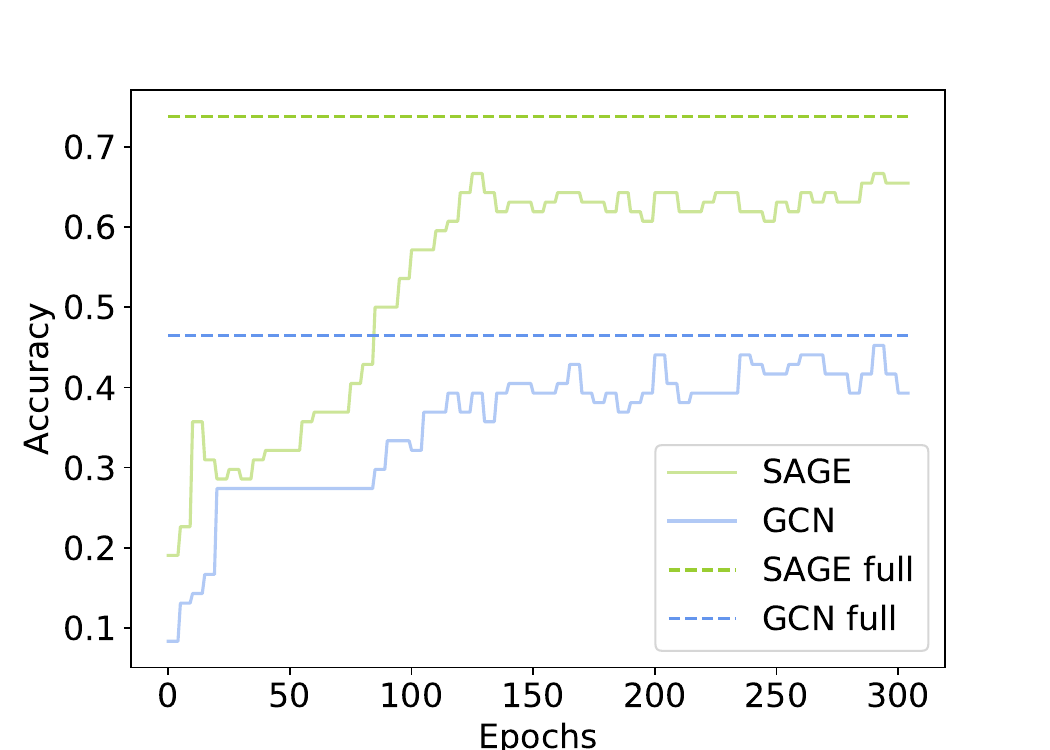}
\includegraphics[width=0.32\linewidth]{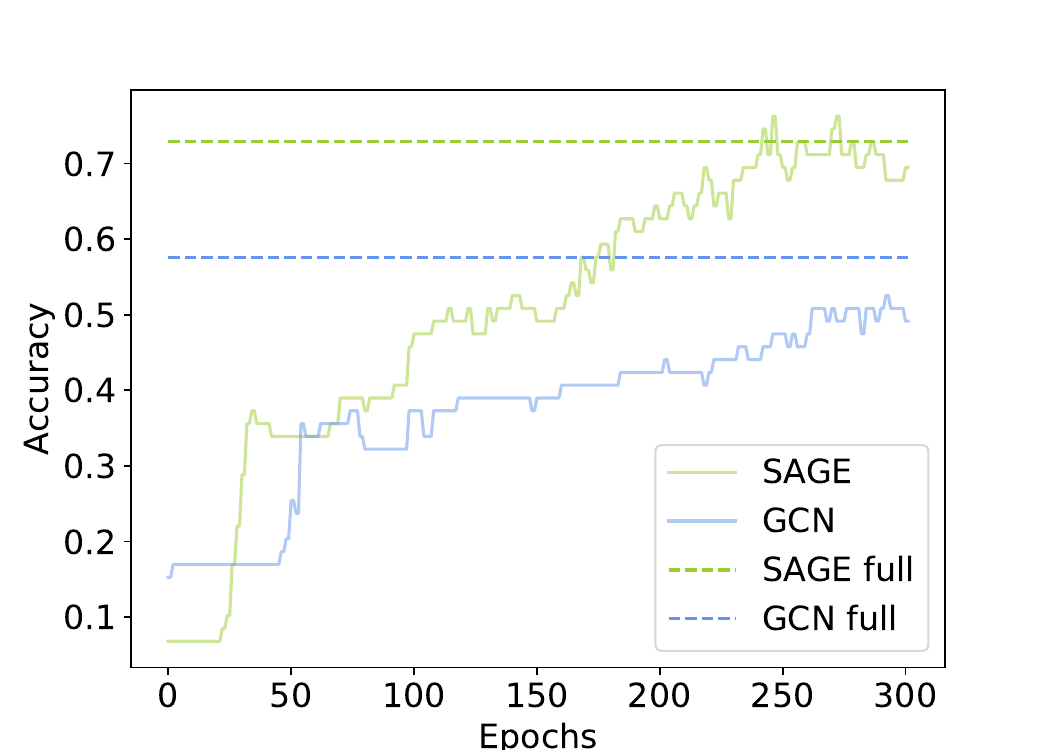}
\includegraphics[width=0.32\linewidth]{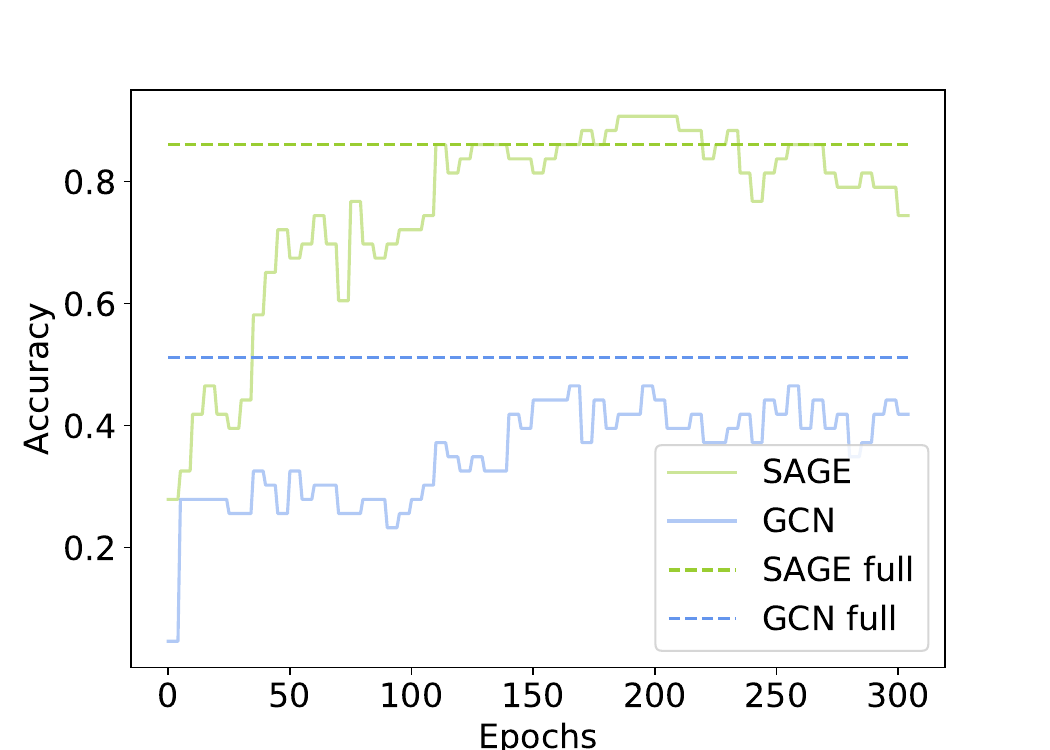}
\includegraphics[width=0.32\linewidth]{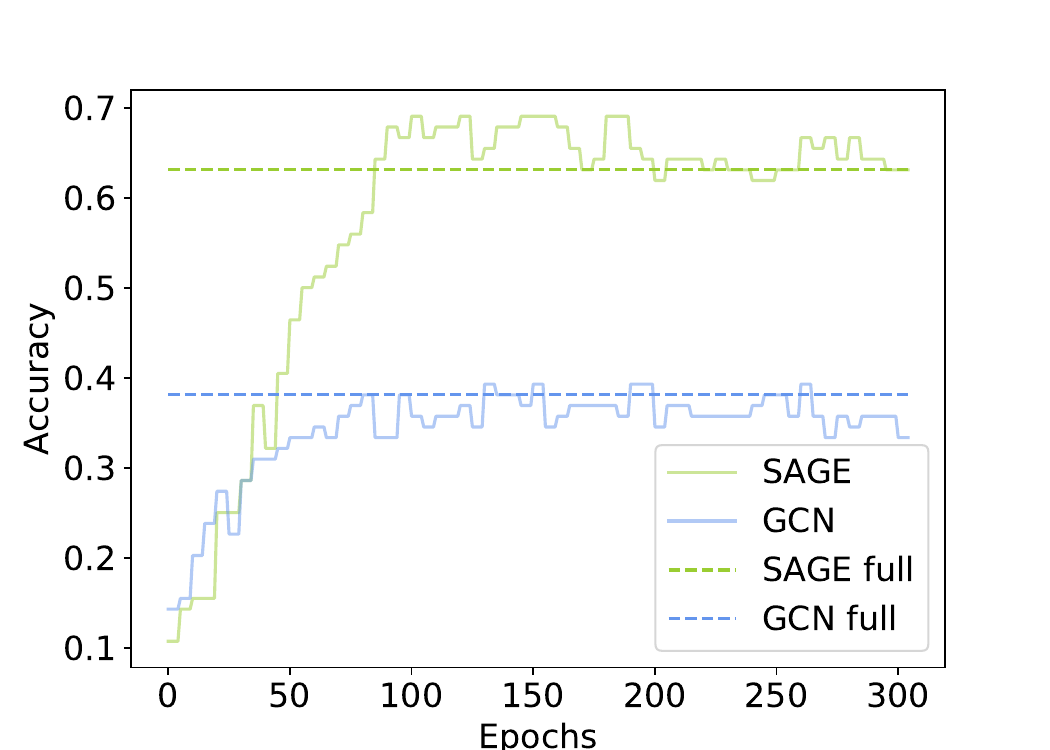}
\includegraphics[width=0.32\linewidth]{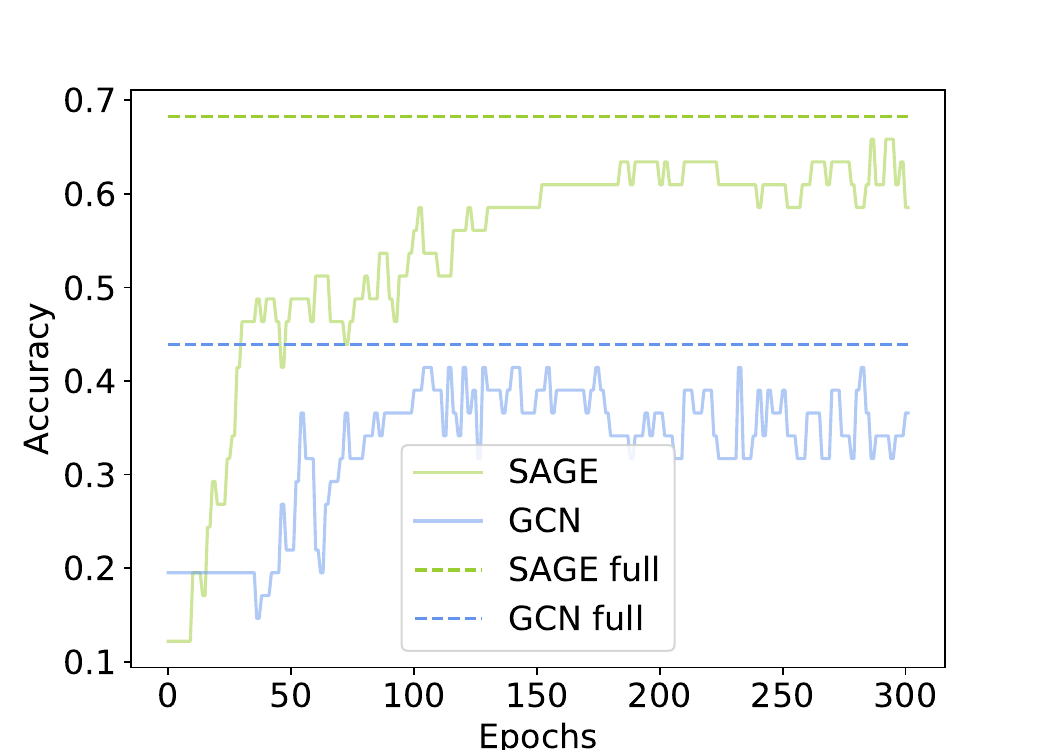}
\caption{Node sampling results for Cora (top) and CiteSeer (bottom) with batch size $32$. We consider three scenarios in terms of the small graph size $n$ and the graph sampling interval $\gamma$: $n=300$, $\gamma=5$ epochs (left); $n=500$, $\gamma=5$ epochs (center); $n=300$, $\gamma=2$ epochs (right). Note that these graphs have size equal to approximately $10$-$16\%$ of the original graph size.}
\label{fig:node_sampling_cora_cite}
\end{figure*}

\noindent \textbf{Node sampling.}~In Figure \ref{fig:node_sampling_cora_cite}, we consider three scenarios in terms of the small graph size $n$ and the graph sampling interval $\gamma$: $n=300$, $\gamma=5$ epochs (left); $n=500$, $\gamma=5$ epochs (center); $n=300$, $\gamma=3$ epochs (right). Note that these graphs have size equal to approximately $10-16\%$ of the original graph size. The GNNs are trained for $300$ epochs with a learning rate $1e{-4}$ and batch size $32$. In Figure \ref{fig:node_sampling_cora_cite}, we observe that resampling graphs every $\gamma=5$ epochs, at either $n=300$ (left) or $n=500$ (center), is not enough to train both models (but especially GraphSAGE) to correctly classify nodes on the full Cora graph under $300$ epochs. However, decreasing the resampling interval to $\gamma=2$ (right) helps. In the case of CiteSeer, $n=300$ and $\gamma=5$ (left) are enough to match the accuracy of the model trained on the full graph, but the models learn faster when $n=500$ (center). Increasing the sampling rate (right) increases the variability in accuracy and worsens performance.

\noindent \textbf{Computational graph sampling.}~In this experiment, we fix the neighborhood sizes for computational graph sampling in both layers at $K_1=K_2=32$, and also consider node sampling with batch size $32$.
The combinations of graph size $n$ and sampling interval $\gamma$, and the number of epochs and the learning rate are the same as in the node sampling experiment. The results are reported in Figure \ref{fig:comp_sampling_cora_cite}. For Cora, we observe the best results for $n=500$ and $\gamma=5$ (center), and for CiteSeer, for $n=300$ and $\gamma=5$ (left). On both datasets, increasing the sampling rate (right) increases the variability in performance, which is undesirable. In the case of CiteSeer, increasing the graph size (center) leads to some overfitting for GraphSAGE.

\begin{figure*}[t] 
\centering
\includegraphics[width=0.32\linewidth]{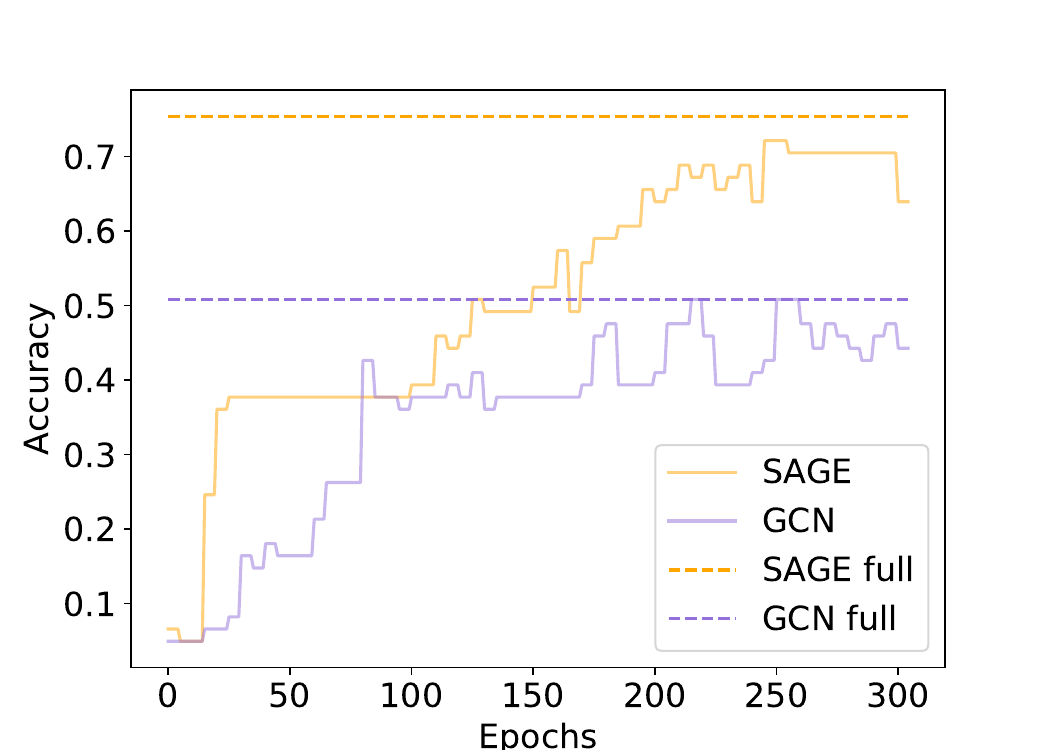}
\includegraphics[width=0.32\linewidth]{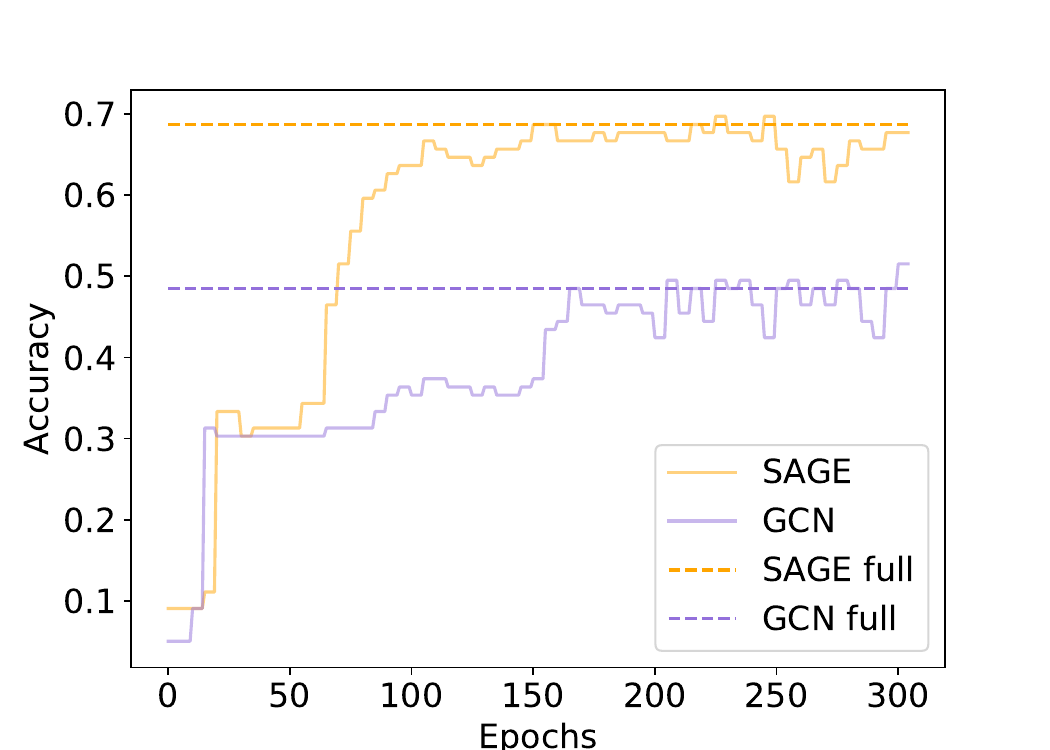}
\includegraphics[width=0.32\linewidth]{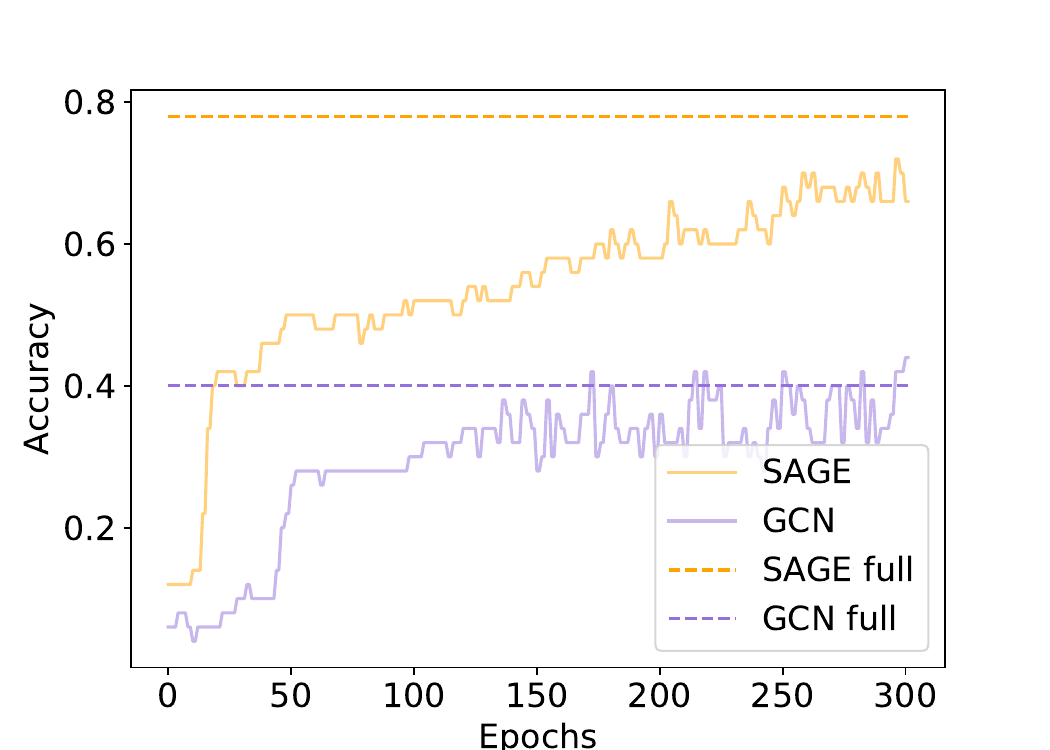}
\includegraphics[width=0.32\linewidth]{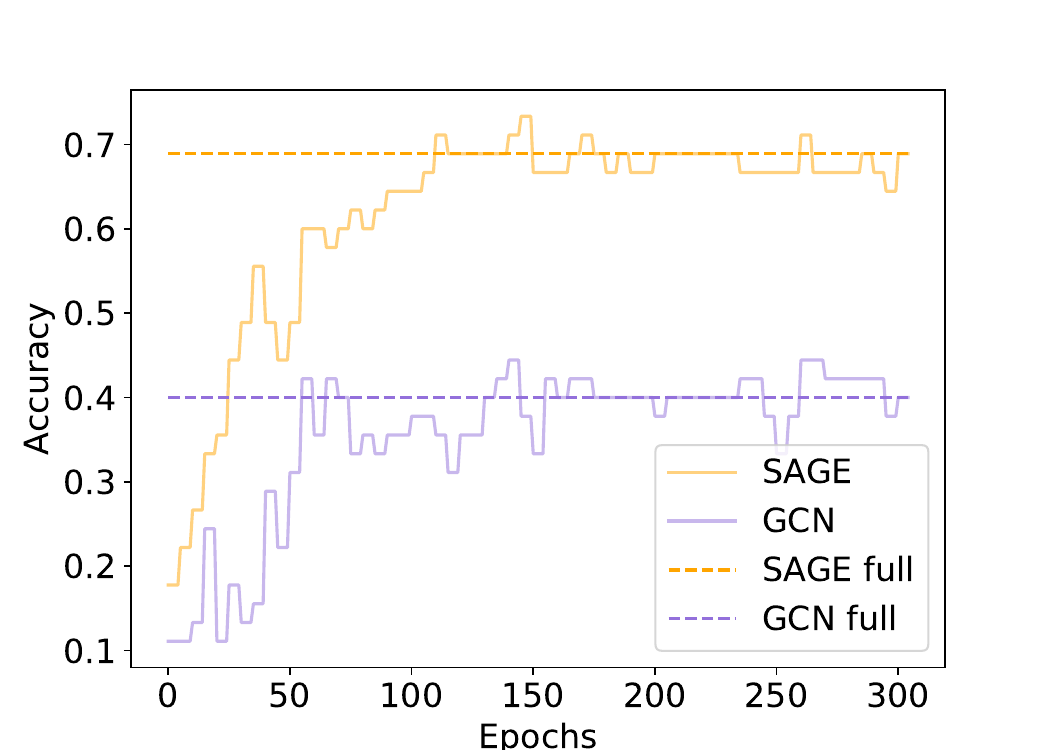}
\includegraphics[width=0.32\linewidth]{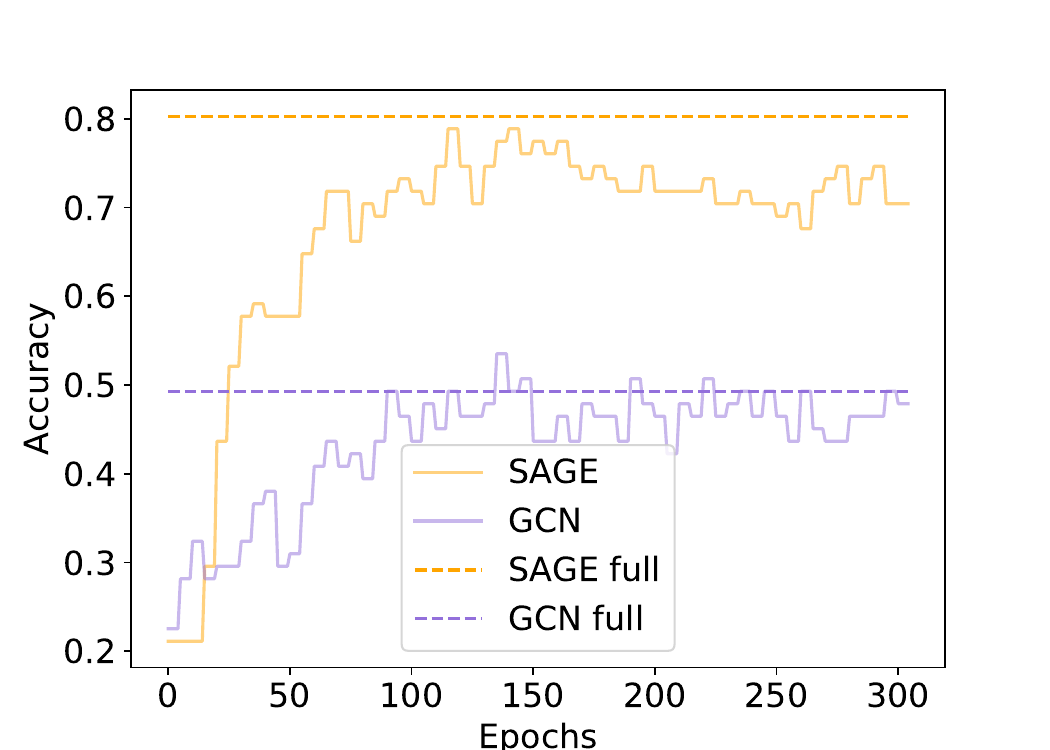}
\includegraphics[width=0.32\linewidth]{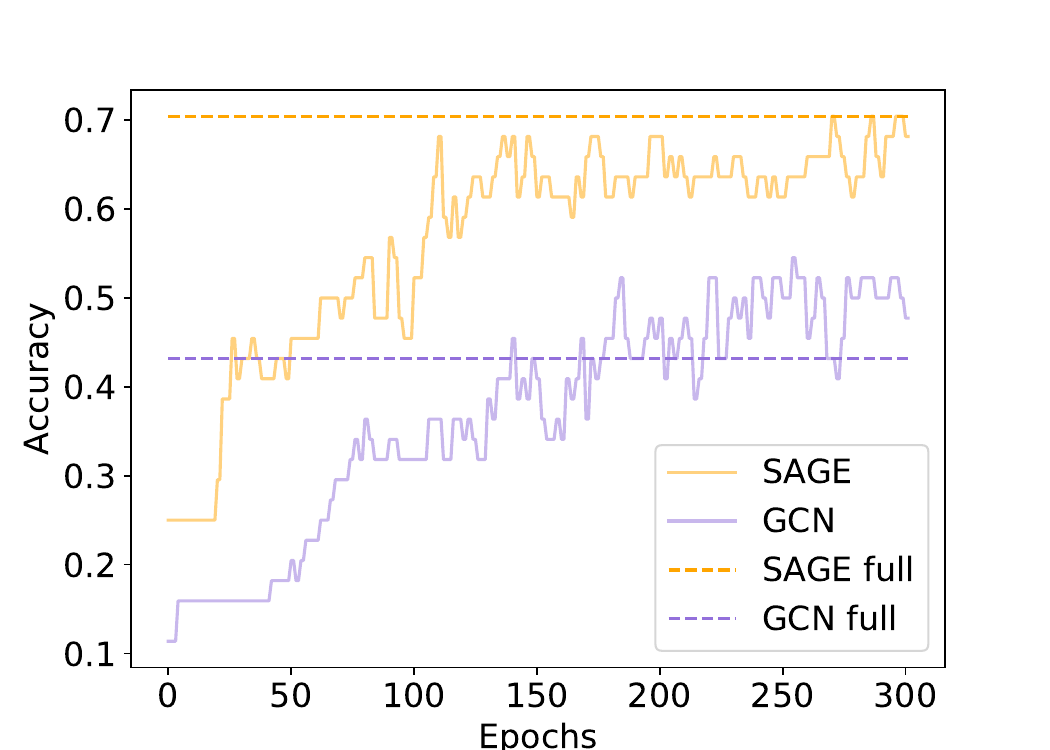}
\caption{Computational graph sampling results for Cora (top) and CiteSeer (bottom) with $K_1=K_2=32$ and batch size $32$. We consider three scenarios in terms of the small graph size $n$ and the graph sampling interval $\gamma$: $n=300$, $\gamma=5$ epochs (left); $n=500$, $\gamma=5$ epochs (center); $n=300$, $\gamma=2$ epochs (right). Note that these graphs have size equal to approximately $10$-$16\%$ of the original graph size.}
\label{fig:comp_sampling_cora_cite}
\end{figure*}

\begin{table}[t]
\caption{Dataset statistics.}
\centering
\begin{tabular}{lcccc}
\hline
         & Nodes ($N$) & Edges & Features & Classes ($C$)  \\
\hline
Cora     & 2708  & 10556 & 1433     & 7       \\
CiteSeer & 3327  & 9104  & 3703     & 6       \\
PubMed   & 19717 & 88648 & 500      & 3       \\
ogbn-mag & 1,939,743 & 21,111,007	& 128 &  349 \\
\hline
\end{tabular}
\label{tab:stats}
\end{table}

\section{Experiments without Node and Computational Graph Sampling} \label{appendix: no_sampling}

\begin{figure*}[t]
\centering
\includegraphics[width=0.32\linewidth]{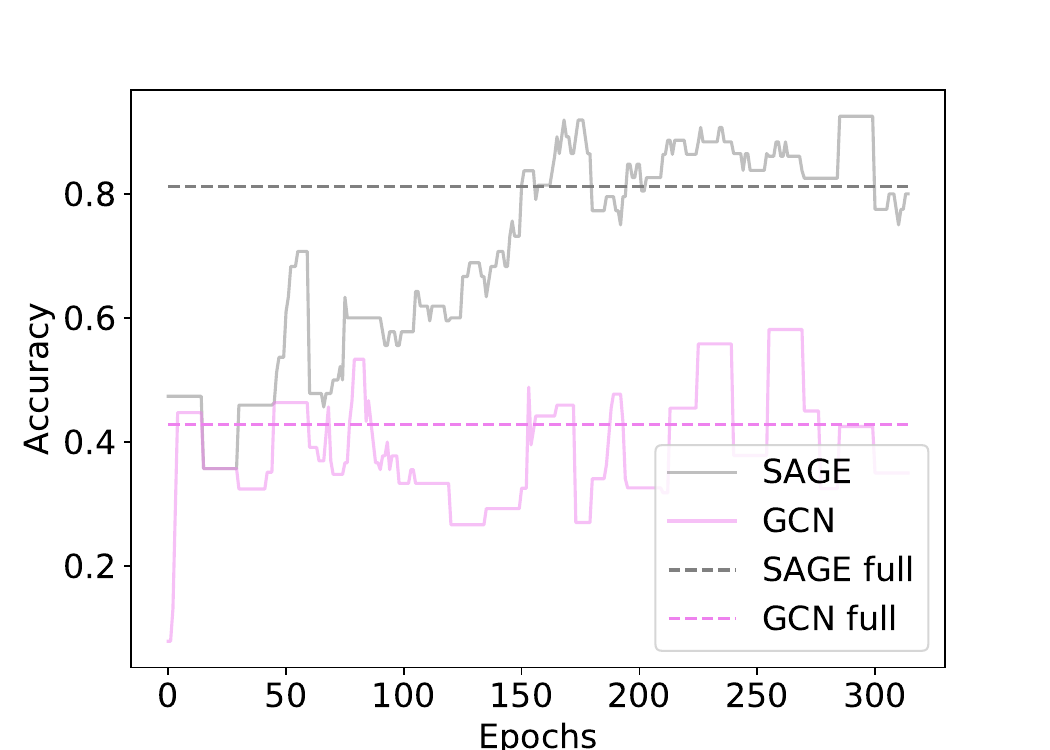}
\includegraphics[width=0.32\linewidth]{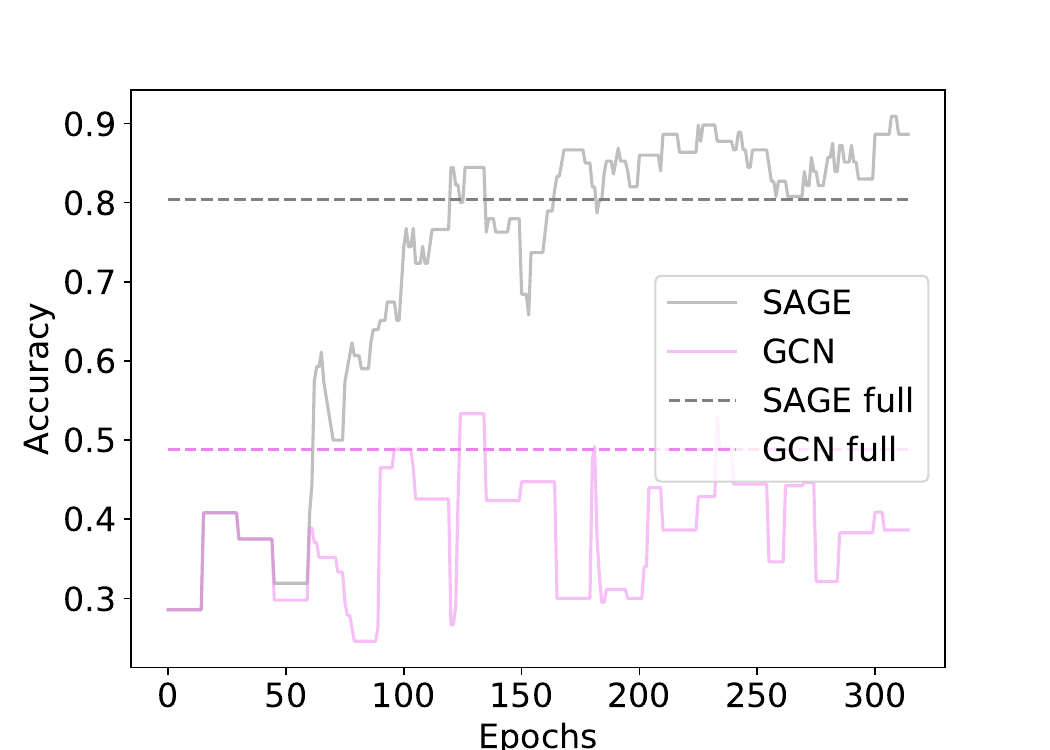}
\includegraphics[width=0.32\linewidth]{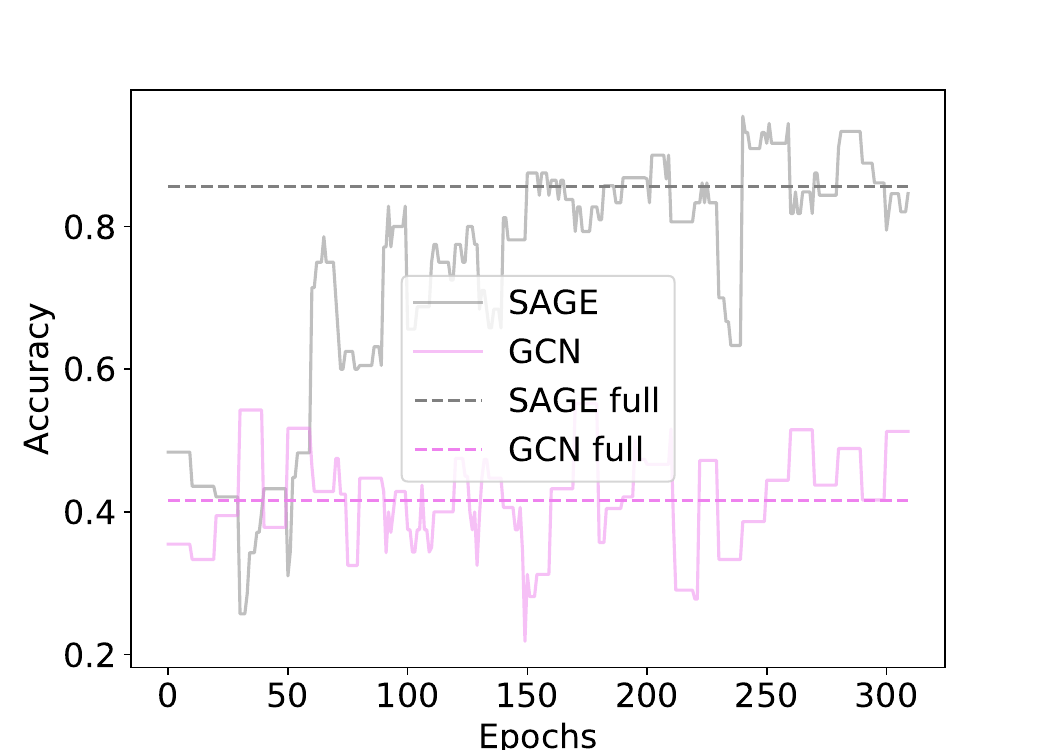}
\caption{Results for PubMed without any sampling other than the graph sequence sampling. We consider three scenarios in terms of the small graph size $n$ and the graph sampling interval $\gamma$: $n=1500$, $\gamma=15$ epochs (left); $n=2000$, $\gamma=15$ epochs (center); $n=1500$, $\gamma=10$ epochs (right). Note that these graphs have size equal to approximately $10\%$ of the original graph size.}
\label{fig:no_sampling}
\end{figure*}

Here, we consider GNNs without any form of sampling other than the $n$-node random graph sequences sampled from the target graph.
%For Cora and CiteSeer, $n=800$ and the GNNs are trained for $150$ epochs with a learning rate $1e{-3}$. For PubMed, $n=3000$ and
The combinations of graph size $n$ and sampling interval $\gamma$ are the same as in the two previous experiments, and the GNNs are trained for $300$ epochs with a learning rate $1e{-3}$. The results are reported in Figure \ref{fig:no_sampling} for PubMed and in Figure \ref{fig:no_sampling_cora_cite} for Cora and CiteSeer. The GraphSAGE models trained on the random graph sequences generally achieve better performance on the random graph sequences than on the target graph, with slight accuracy improvement when $n$ is increased and higher variability in accuracy when $\gamma$ is decreased. The GCN performance is subpar in both cases and for all combinations of $n$ and $\gamma$. We observe more variability in accuracy than in Figures \ref{fig:node_sampling} and \ref{fig:comp_sampling}, which is expected since in the absence of node sampling, the gradient updates are calculated at all nodes; and in the absence of computational graph sampling, the effective graph neighborhoods are less regular. It is also interesting to note that gradient and computational graph sampling provide good inductive bias in this experiment, as the test accuracy achieved by the GNNs with node and computational graph sampling, in Figures \ref{fig:node_sampling},\ref{fig:node_sampling_cora_cite} and \ref{fig:comp_sampling},\ref{fig:comp_sampling_cora_cite}, are higher than those achieved by the GNNs without sampling in Figures \ref{fig:no_sampling}--\ref{fig:no_sampling_cora_cite}.

\begin{figure*}[t]
\centering
\includegraphics[width=0.32\linewidth]{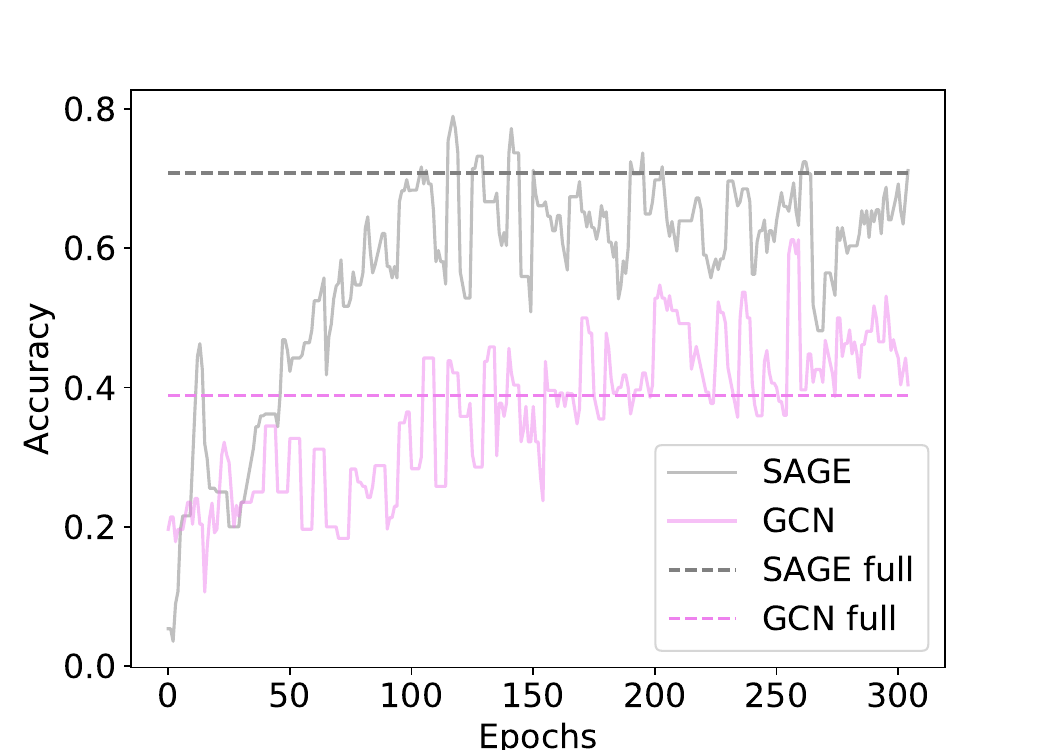}
\includegraphics[width=0.32\linewidth]{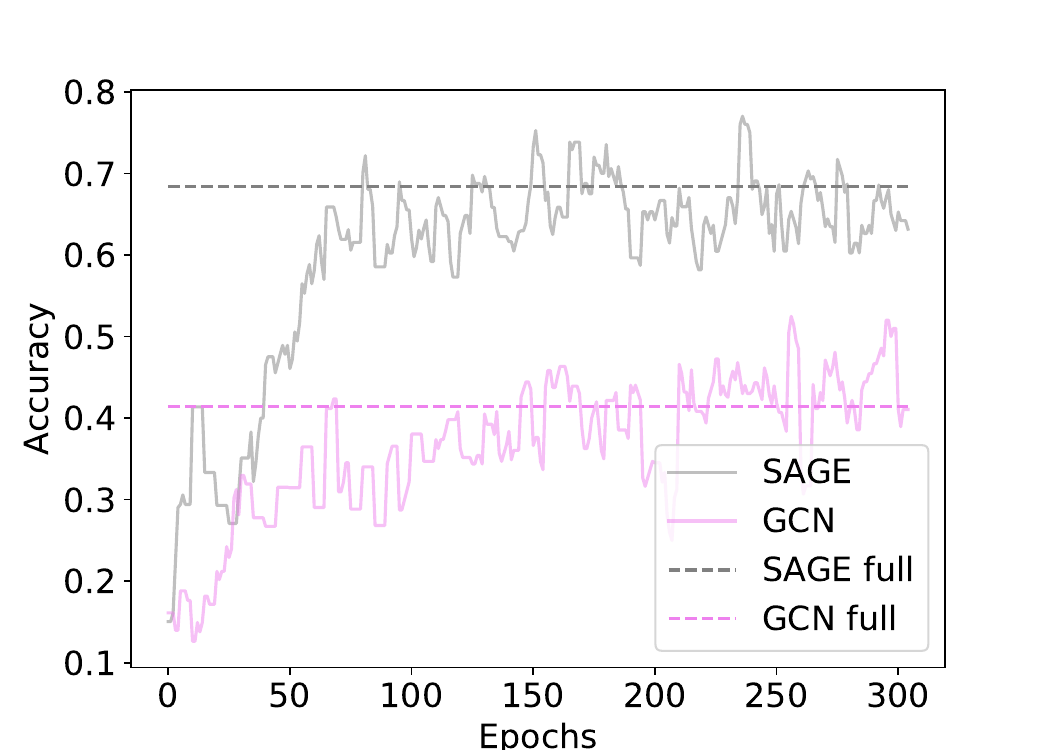}
\includegraphics[width=0.32\linewidth]{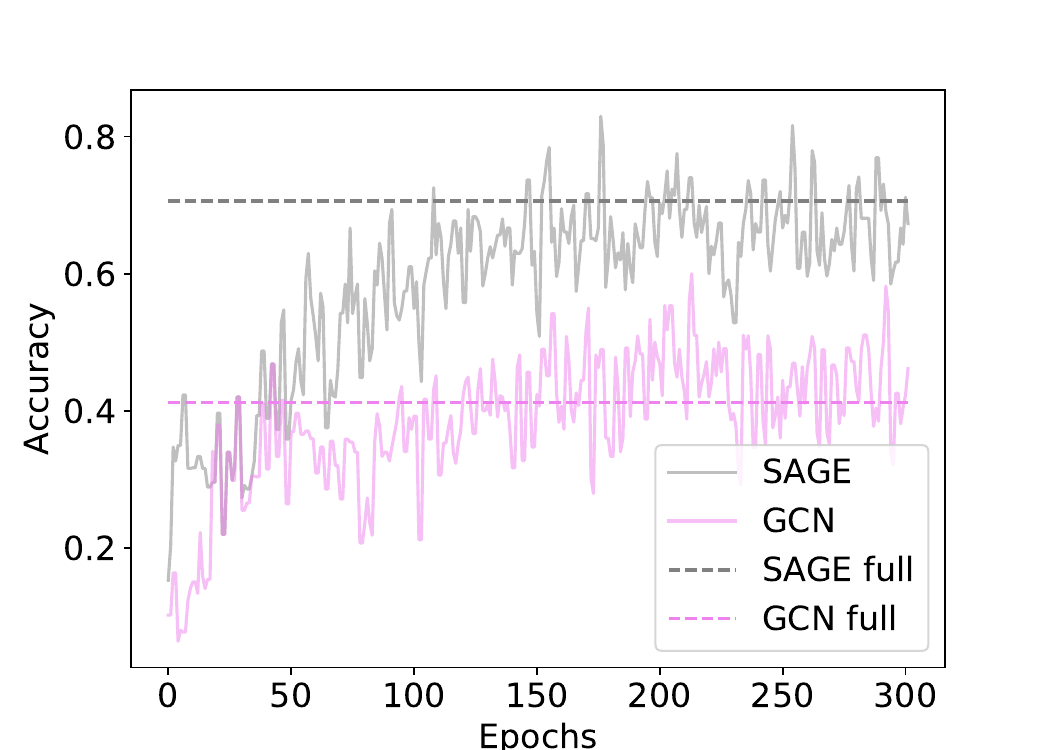}
\includegraphics[width=0.32\linewidth]{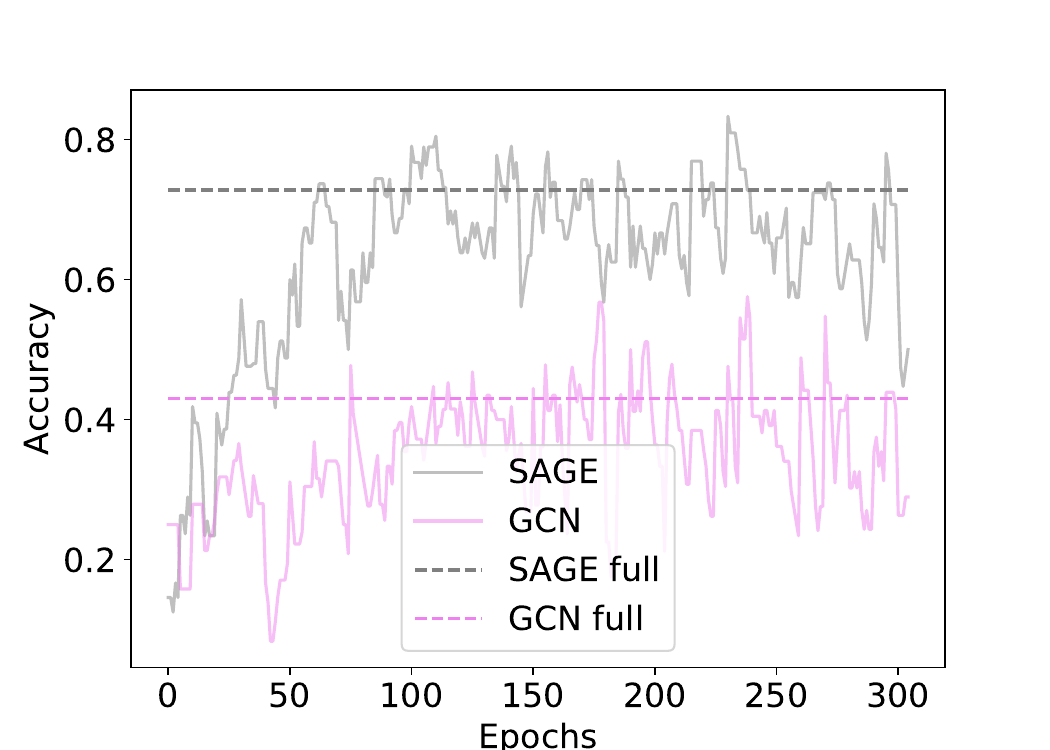}
\includegraphics[width=0.32\linewidth]{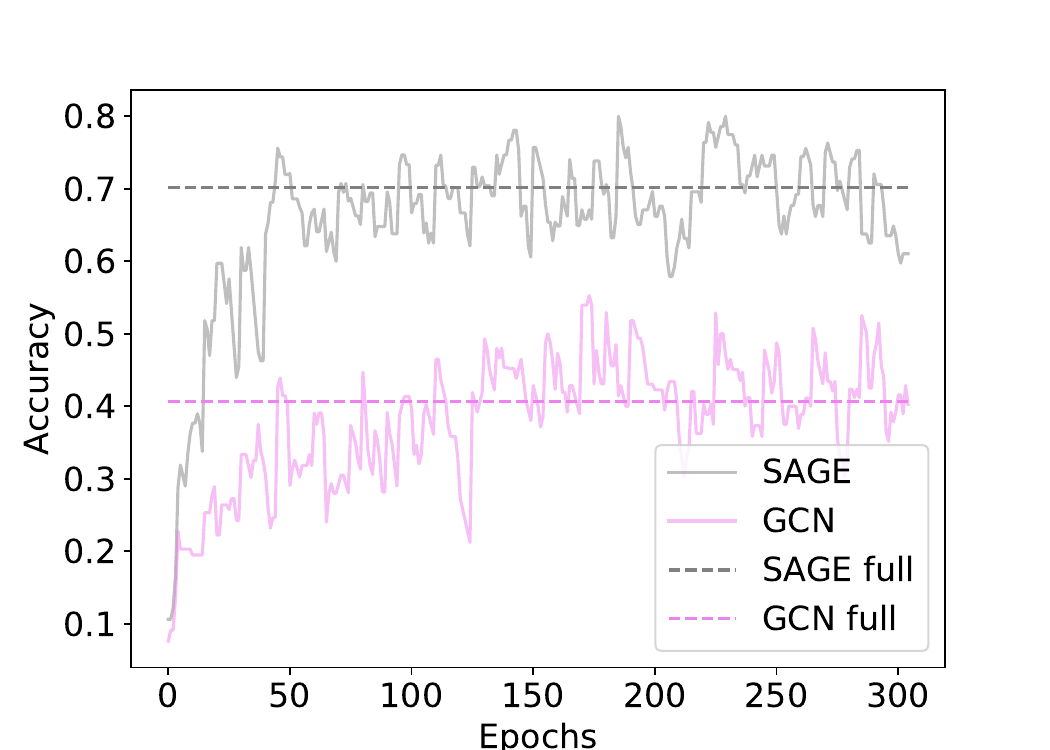}
\includegraphics[width=0.32\linewidth]{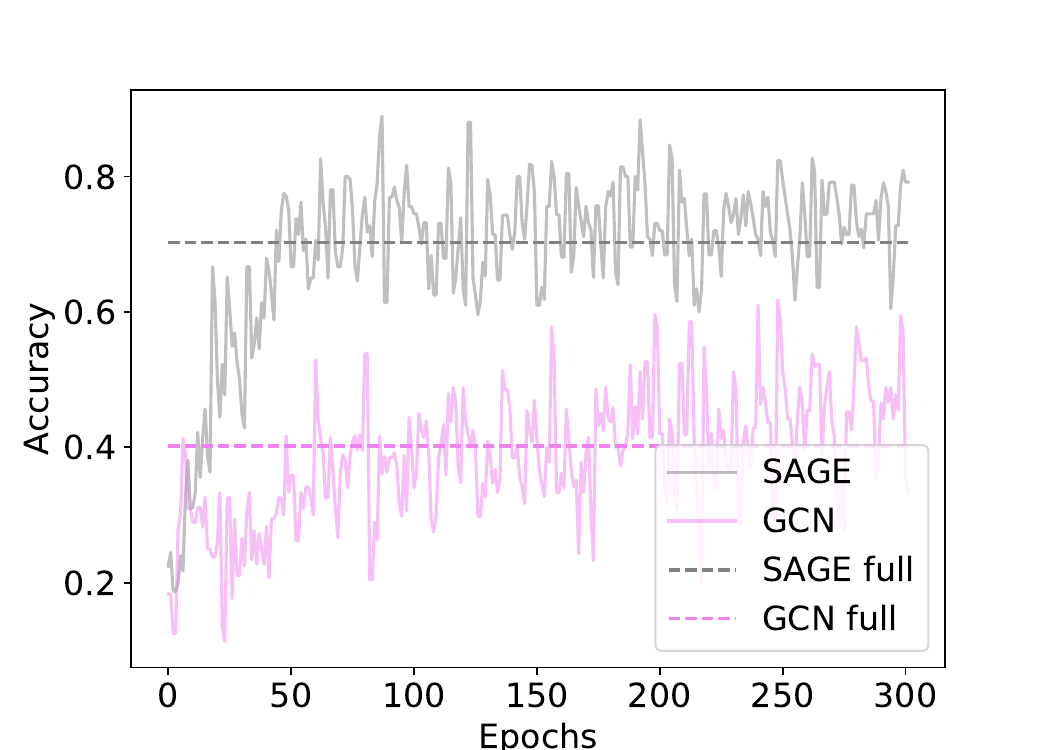}
\caption{Results for Cora (top) and CiteSeer (bottom) without any sampling other than the graph sequence sampling. We consider three scenarios in terms of the small graph size $n$ and the graph sampling interval $\gamma$: $n=300$, $\gamma=5$ epochs (left); $n=500$, $\gamma=5$ epochs (center); $n=300$, $\gamma=2$ epochs (right). Note that these graphs have size equal to approximately $10$-$16\%$ of the original graph size.}
\label{fig:no_sampling_cora_cite}
\end{figure*}

\end{document}

%% file: gcn_sampling.tex
\section{Sampling-Based GNNs}\label{sec: sampling GCN}

\subsection{GNNs: Preliminaries}

%Fundamentally, GNNs work just like any other neural network. We start with some input data, pass it through multiple layers of neurons that are connected in certain ways, then optimize against a loss function to get an output representation. The difference with a graph neural network is that the computation graph we pass depends on the underlying graph structure. In GNNs, nodes iteratively update their representation by exchanging information with their neighbors using a `learnable' message-passing paradigm. 

 Let $G_n=(V(G_n),E(G_n))$ be a graph where $V(G_n)$, $|V(G_n)|=n$, is the set of nodes, $E(G_n) \subseteq V(G_n) \times V(G_n)$ is the set of edges, and $\mathbf A \in \mathbb R^{n \times n}$ is the adjacency matrix. Let $\mathbf X \in \mathbb R^{n \times F}$  be the matrix of input features, where $\mathbf x_v$ is the $F$-dimensional feature vector of node $v$.

 In its most general form, a GNN consists of $L$ layers, each of which composes an \emph{aggregation} of the information in each node's neighborhood, and a \emph{combination} of the aggregate with the node's own information. Explicitly, for each layer $\ell+1$ and node $v$ we can write the following propagation rule \citep{xu2018GIN}
\begin{align} \label{eqn:propagation}
\begin{split}
\mathbf{h}^{(\ell+1)}_{N(v)}&=\texttt{AGGREGATE}_\ell\left(\{\mathbf h^{(\ell)}_u,u\in N(v)\}\right) \\
\mathbf h^{(\ell+1)}_v& =  \texttt{COMBINE}_\ell\left(\mathbf{h}^{(\ell)}_{v},\mathbf{h}^{(\ell+1)}_{N(v)}\right)\qquad \text{ and }\qquad \mathbf h^{(0)}_v = \mathbf x_v,
\end{split}
\end{align}
 \comment{
 \todo{just wrote one version, but revise later}
 \begin{align} \label{eqn:propagation}
\mathbf{h}^{(\ell+1)}_{N(v)}&=\texttt{AGGREGATE}_\ell(\{\mathbf h^{(\ell)}_u,u\in N(v)\cup\{ v\}\}) \\
\mathbf h^{(\ell+1)}_v& =  \sigma(W^{(\ell+1)}\mathbf{h}^{(\ell+1)}_{N(v)})\qquad \text{ and }\qquad \mathbf h^{(0)}_v = \mathbf x_v,
\end{align}
 }
where $N(v)$ is the neighborhood of node $v$ and $\mathbf H^{(\ell)}\in\mathbb R^{n\times F_\ell}$ is the embedding produced by layer $\ell$.
%, which is also the input to layer $\ell+1$.

In so-called node-level or inductive learning tasks, the GNN output is ${\mathbf Z=\mathbf H^{(L)}}$.
In graph-level or transductive learning tasks, the GNN has a final layer called the \emph{readout} layer, which aggregates the node embeddings $\smash{\mathbf h_v^{(L)}}$ into a single graph embedding $\smash{\mathbf h_G \in \mathbb R^{F'}}$ as follows
\begin{equation} \label{eqn: readout}
\mathbf h_G = \texttt{READOUT}(\mathbf h_v^{(L)}, v \in V(G)) 
\end{equation} 
and where $F'$ is the embedding dimension. The output of the GNN is then $\mathbf Z = \mathbf h_G$. The \texttt{READOUT} function can be a fully connected layer over the graph nodes, a sequence of pooling layers \citep{ying2018hierarchical,zhang2018end}, or a simple aggregation operation such as the maximum, the sum, or the average \citep{xu2018GIN}. 
We focus on architectures where the \texttt{READOUT} layer (if present) is a simple aggregation operation, such as in \citep{xu2018GIN} and \citep{dwivedi2020benchmarking}. This is a common assumption in the literature, since such aggregations are invariant to node relabelings and prevent the number of learnable parameters from depending on the size of the graph.

For examples of how two popular GNN architectures, GCN and GraphSAGE, can be written in the form of \eqref{eqn:propagation}--\eqref{eqn: readout}, refer to Appendix \ref{appendix: gcn_sage}. We will experiment with these architectures in Section \ref{sec: experiment}.

\textbf{Training.}~We consider both unsupervised and supervised learning tasks.
In the unsupervised case, the goal is to minimize a loss $\mathcal L$ over the dataset $\mathcal T = \{\mathbf X_m\}_m$. In the supervised case, each data point is additionally associated with a label $\mathbf Y_m$, which factors in the computation of the loss. In either case, the goal is to solve the following optimization problem
\begin{equation}\label{eqn:unsup_loss}
    \min_{\mathbf W^{(\ell)}, 1 \leq \ell \leq L}\frac{1}{|\mathcal T|}\sum_{m=1}^{|\mathcal T|}\mathcal L(\mathbf Z_m)
\end{equation}
where $\mathbf Z_m$ are the outputs of the GNN corresponding to the inputs $\mathbf X_m$.

\comment{
In the unsupervised case, the goal is to minimize a loss $\mathcal L$ over the dataset $\mathcal T = \{\mathbf X_m\}_m$, i.e.,
\begin{equation}\label{eqn:unsup_loss}
    \min_{\mathbf W^{(\ell)}, 1 \leq \ell \leq L}\frac{1}{|\mathcal T|}\sum_{m=1}^{|\mathcal T|}\mathcal L(\mathbf Z_m)
\end{equation}
where $\mathbf Z_m$ are the outputs of the GNN corresponding to the inputs $\mathbf X_m$.
In the supervised case, we are given a training set $\mathcal T = \{\mathbf X_m,\mathbf Y_m\}_m$, where $\mathbf Y_m$ are the ground-truth labels, and the goal is to solve the empirical risk minimization problem (ERM)
\begin{equation}\label{eqn:sup_loss}
    \min_{\mathbf W^{(\ell)}, 1 \leq \ell \leq L}\frac{1}{|\mathcal T|}\sum_{m=1}^{|\mathcal T|}\mathcal L(\mathbf Y_m,\mathbf Z_m) \text{.}
\end{equation}
}
This problem is solved using a gradient descent approach with updates
\begin{equation}\label{eqn:GD}
    \mathbf{W}^{(\ell)}_{t+1}=\mathbf{W}^{(\ell)}_t-\frac{\eta}{|\mathcal T|} \sum_{m=1}^{|\mathcal T|}\nabla_{\mathbf W} \mathcal L(\mathbf Z_m)
\end{equation}
%We focus on node-prediction applications of GCNs. For this purpose, the goal is to learn the weight matrices $\mathbf W$ to minimize a loss function:
%\begin{equation}\label{eqn:loss}
%    \mathcal L=\frac{1}{|\mathcal Y|}\sum_{v\in\mathcal Y}\loss(y_v,z_v),
%\end{equation}
%where $y_v$ is the ground truth label for node $v$ and $\mathcal Y$ is the set of all labeled nodes. Then the loss is minimized by finding a set of coefficients $\mathbf W$ using  a gradient descent approach, i.e., repeatedly, update weights in the opposite direction of weights:
%\begin{equation}\label{eqn:GD}
%    \mathbf{W}_{t+1}=\mathbf{W}_t-\eta \nabla_W \mathcal L,
%\end{equation}
where $\eta$ is the step size or learning rate. The process stops when the gradient becomes smaller than some predetermined small constant $\smash{\sum_{m}|\nabla_{\mathbf W} \mathcal L|}\leq |\mathcal T| \epsilon$. 

Computing the loss \eqref{eqn:unsup_loss} and the gradients \eqref{eqn:GD} can be cumbersome when the graph $G$ is large, but as we discuss in the next section, both can be estimated using sampling techniques.

%%% Sampling-based GNNs %%%

\subsection{Sampling-Based GNNs: An General Framework}
% \todo{Y:maybe the following is repetitive?}
%The demand to scale GNNs 
%is increasing as the graph data sizes continue to grow. At the same time, training GNNs on large-scale graphs requires significant computational resources and memory usage. To overcome these challenges, recent research employs diverse sampling techniques. 

% {\bf Amin:} The paragraph below is confusing. Are there two sampling appraoches? Then why is there a third? :) Are we justifying GraphSAGE FASTGCN etc? Or suggest a new way of training (on smaller graphs or subgraphs?) Once we decide this, I can rewrite & shorten the this page. }

We consider a variety of sampling techniques under a general unified framework. In particular, we offer a formalization for sampling-based GNNs
that utilize sampling in one or both of two  ways:
node sampling and computational graph sampling. 
% Further, we use a third sampling scheme that draws a small subgraph from the input at each iteration of training the sampling-based GNN architecture.  
% (I) sampling subgraphs from a large input graph for training, and (II) 
% employing sampling schemes within the GNN framework, which is trained on the subgraph. The latter sampling approach is comprised of two key components: node sampling and computational graph sampling.  
% Specifically, node sampling entails selecting nodes to estimate the gradient and update the weights in each iteration.   Computational-graph sampling involves  pruning the computational graph to determine the node embeddings.  
%Next, we provide a detailed description of each sampling technique and a formal abstraction of the unified framework.

% To provide a more general framework that abstracts away from the details of specific architectures, we formalize a class of algorithms based on sampling. Sampling is a key technique used in this class of algorithms, typically employed in two steps. First, nodes are sampled to estimate the gradient and update the weights for the next iteration. We refer to this as \emph{node sampling}. Second, a sampling method is used to prune the computational graph in computing the node embeddings. We refer to this as \emph{computational graph sampling}. We describe each of these types of sampling next, and then we make a formal abstraction of sampling-based GNNs.

% Node sampling

\noindent\textbf{Node sampling.}
% \todo{Y: our result applies to SGD only. We can potentially remove this section, but I kept it for now, since its still a nice abstractoin}
Due to the difficulty of computing gradient descent steps for every node on a large graph, a common technique to accelerate GNN training is to perform stochastic gradient descent (SGD) over minibatches of graph nodes. SGD in its conventional form samples a minibatch of nodes $V_B\subset V(G)$ and then uses the gradient on these nodes to estimate \eqref{eqn:GD}. I.e., it uses $\smash{\frac{1}{|V_B|}\sum_{v\in V_B}\nabla_{\mathbf W}\mathcal L(\mathbf z_v)}$ as an estimator for $\nabla_{\mathbf W}\mathcal L(\mathbf Z)$. 

Other variants of SGD employ importance sampling to estimate the gradient. %(e.g., GraphSAINT \citep{zeng2019graphsaint}). 
Let $\nu_g:G\to\mathbb R^+$ be a weight function influencing the sampling probability as $\smash{\frac{\nu_g(G,v)}{|V(G)|}}$. Then,
\begin{equation}\label{eq: loss estimate weighted}
 \nabla_{\mathbf W}\tilde{\mathcal L}_{\nu_g}(\mathbf Z)=\frac{1}{|V_B|}\sum_{v\in V_B}\frac{1}{\nu_g(v)} \nabla_{\mathbf W} \mathcal L(\mathbf z_v)
\end{equation}
gives an estimator for \eqref{eqn:GD}. Note that if $\nu_g(v)=1$ for all $v \in V(G)$, we recover conventional SGD.

% Computational graph sampling

\noindent\textbf{Computational graph sampling.}
% \todo{make it more precise}
In conventional GNN training, i.e., without sampling, the size of the computational graph grows exponentially with the number of layers. Hence, many architectures prune the computational graph by sampling which neighbor-to-neighbor connections to keep (those that aren't sampled are discarded). Specifically, a computational graph sampler $\nu_C$ taking a graph $G$ and a node $v\in V(G)$ as inputs outputs a sampled computational graph denoted $\nu_C(G,v)$ (or simply $\nu_C(v)$ if $G$ is clear from context). This sampled computational graph, $\nu_C(G,v)$, then replaces the full computational graph in the forward propagation step \eqref{eqn:propagation}.

Having established these sampling processes, we now proceed
to describe our proposed unified algorithmic framework which uses a node sampler to compute gradients and a computational graph sampler to compute the forward pass.
% With this sampling process established, we proceed to describe the unified algorithmic framework for sampling-based GNNs.% and then

% \todo{\bf Amin: the above paragraph needs to be rewritten.}
% \red{L: I think we can probably add an equation for the probability of adding a node $u$ to the computational graph of $v$ here, right? Here is a suggestion:
% \begin{equation*}
%     (u,v) \sim \frac{\nu_C(v,u)}{|N(v)|} 
% \end{equation*}
% }

% The algorithm

\noindent\textbf{A unified algorithmic framework for sampling-based GNNs.}
A sampling-based GNN takes as inputs a graph, a gradient sampler $\nu_g$, and a computational graph sampler $\nu_C$. It then samples a minibatch of nodes $V_B$ from the graph using $\nu_g$; uses $\nu_C$ to sample a computational graph $G_V$; performs forward propagation on $G_V$; and outputs the embeddings $\mathbf{Z}$ of the sampled nodes $V_B$. This is shown in Algorithm~\ref{alg: sampling-GNN}, where we write the outer for loop for explanation purposes only. In practice, the steps in this outer loop are executed in parallel. 
%The algorithm can have an additional \texttt{READOUT} layer for graph-level learning but for simplicity we only state the inductive learning.

This framework encompasses many well-known GNN architectures. For instance, in GraphSAGE \citep{hamilton2017inductive}, the node sampler $\nu_g$ assigns a uniform probability to all nodes, and the computational graph sampler $\nu_C$
draws a fixed number of neighbors for each node.
%, as GraphSAGE is commonly implemented with SGD. 
A second example is FastGCN, which prunes the computational graph by choosing neighbors proportionally to the normalized adjacency matrix entries. 
We provide a more detailed discussion on how various architectures, including GraphSAGE, FastGCN and shaDoWGNN, fit into this algorithm in Appendix~\ref{sec:app}.

\begin{minipage}{.51\textwidth}
\begin{algorithm}[H]%[htb!]
\footnotesize
   \caption{\footnotesize Unified Framework for Sampling-Based GNN}
   \label{alg: sampling-GNN}
    %\begin{algorithmic}
   {\bfseries Function} \texttt{SamplingBasedGNN} 
   
   {\bfseries Input:} graph $G_t$; gradient sampler $\nu_g$; comp. graph sampler $\nu_C$ %depth $L$.
   %aggregator functions $\texttt{AGGREGATE}_\ell, \texttt{COMBINE}_\ell$ for $\ell\in\{0,\dots,L\}$; depth $L$.
   %\State {\bfseries Output:} Embedding of sampled nodes. 
   
   %\State Sample minibatch of nodes $V_B$ from  $V(G_t)$ proportional to weights $\nu_g$.
   \BlankLine
   Sample minibatch of $|V_B|$ nodes $v \sim \nu_g(v)$ 

   \For{$v\in V_B$}{
   
   Sample computational graph $G_v \sim \nu_C(v)$
   
       \For{$\ell =0$ to $L-1$}{ 
       $\mathbf{h}^{(\ell+1)}_{N(v)}=\texttt{AGGR}_\ell(\{\mathbf h^{(\ell)}_u,(u,v)\in E(G_v)\}$
       
       $\mathbf h^{(\ell+1)}_v =  \texttt{COMBINE}_\ell\Big(\mathbf{h}^{(\ell+1)}_{N(v)}, \mathbf{h}^{(\ell+1)}_{v}\Big)$
        }
    
        $\mathbf Z = \mathbf H^{(L)}$ or $\texttt{READOUT}(\mathbf h_v^{(L)}, v \in V(G))$ 
        %\State $\mathbf z_v=\mathbf h^{(L)}_v$ 
    }
    %\State $\nabla_{\mathbf W_t}\tilde{\mathcal L}_{\nu_g}(\mathbf Z)=\frac{1}{|V_B|}\sum_{v\in V_B}\frac{1}{\nu_g(v)} \nabla \mathcal L(\mathbf z_v).$
    %\State  \textbf{return} $\{\mathbf z_v, \text{for } v\in V_B\}$
    %\end{algorithmic}
     \Return{$\mathbf Z$}
\end{algorithm}
\end{minipage}
\hfill
\begin{minipage}{0.49\textwidth}
\begin{algorithm}[H]%[htb!]
    \footnotesize
    \caption{\footnotesize Training by Sampling Local Subgraphs}
    \label{alg: training}
    %\begin{algorithmic}
    {\bfseries Input:} %A subgraph sequence $\{G_n\}_{n\in \mathbb N}$, 
    sample size $N_\epsilon$; subgraph sampler $\mu_S$
    %, \texttt{SamplingBasedGNN} 
    %{\color{gray}(gradient sampler $\nu_g$; computational graph sampler $\nu_C$; aggregator functions $\texttt{AGGREGATE}_\ell, \texttt{COMBINE}_\ell$ for $\ell\in\{0,\dots,L\}$; depth $L$)}.
    \BlankLine
    \While{$|\nabla\mathcal L|>\epsilon$}{
    Draw $N_\epsilon$-node graph $G_t \sim \mu_S$
    
    $\mathbf Z=$\texttt{SamplingBasedGNN}($G_t$)
         
    $\nabla_{\mathbf W_t}\tilde{\mathcal L}_{\nu_g}(\mathbf Z)=\frac{1}{|V_B|}\sum_{v\in V_B}\frac{1}{\nu_g(v)} \nabla \mathcal L(\mathbf z_v)$
   
    $\mathbf{W}_{t+1}=\mathbf{W}_t-\eta \nabla_{\mathbf W_t} \nabla\tilde{\mathcal L}_{\nu_g}$
    
    $t\leftarrow t+1$
    }
%\end{algorithmic}
\end{algorithm}
\end{minipage}
% \end{figure}

\noindent\textbf{Training by sampling local subgraphs.}
Typically, sampling-based GNNs are trained for a fixed number of rounds where each round consists of running Algorithm~\ref{alg: sampling-GNN} followed by a backward pass (i.e., weight updates via some variant of gradient descent; see full description in Algorithm~\ref{alg: classic} in the appendices).  %Therefore, the input graph to Algorithm~\ref{alg: sampling-GNN} remains unchanged during training.
We study a slight modification in which, instead of training the GNN on the full graph, we train it on a collection of smaller subgraphs. This is achieved by employing a subgraph sampler $\mu_S$ which acts as an oracle: it subsamples graphs that are then passed to the sampling-based GNN. %This can be applied to different settings, such  as when we have many graphs  or a single large graph in the training set. For example, 

%In the case of transductive learning, where we have access to many small graphs in training, $\mu_S$ would draw one graph from the training set at each iteration. Alternatively, if the learning task is for a single large graph input, $\mu_S$ samples small subgraphs from the original large input graph.  Then the sampling-based GNN (Algorithm~\ref{alg: sampling-GNN})  takes a new subgraph in input at each iteration.
 % Training terminates the iterations once the gradient magnitude falls below a certain threshold.
% The  subgraph sampler $\mu_S$ can represent either when we have many small graphs in the training set  or just one large input graph. In the former case, such as in the Protein-Protein interactions \cite{},  $\mu_S$ would choose one graph at each iteration. In the latter case, i.e., when the learning task is for one large graph input, $\mu_S$ could be defined as a sampler that draws small subgraphs from the original large input graph.  Then the sampling-based training (Algorithm~\ref{alg: training})  visits a new graph at each iteration, runs GNN architecture (as in Algorithm~\ref{alg: sampling-GNN})  only on the current graph and terminates the iterations once the gradient magnitude falls below a certain threshold.
This change in the training procedure allows analyzing the training convergence of sampling-based GNNs without significantly changing the original algorithm.
Indeed, we will show that, above a certain lower bound on the size of the subgraphs produced by $\mu_S$, Algorithm \ref{alg: training} reaches the neighborhood of a local minimum in a finite number of training steps.

%% file: gcn_limit.tex
\section{Graph Limits and Limit GNNs}\label{sec: GL}

We introduced a training procedure for sampling-based GNNs which consists of training them on a collection of local subgraphs that are sampled at regular intervals during training. The purpose of this change in training procedure is to approximate the effect of node and computational graph sampling. However, some questions remain unanswered. Specifically, does this approach yield results similar to training on the entire input graph? Do the sampled subgraphs contain enough information about the graph, such that training is not significantly affected? If we are able to give positive answers to these questions, we can analyze sampling-based GNNs on smaller graphs, which in turn allows for better interpretability and can facilitate the design and tuning of GNNs by practitioners.
To answer these questions, we turn to the theory of graph limits. We first introduce graph limits and the associated notion of convergence, before defining limit GNNs.

\subsection{Graph Limit Theory}\label{sec: GL def}

% Local convergence provides a way to understand how the local structure of a sequence of graphs changes as the size of the graphs increases.
% Graph convergence allows us to analyze the behavior of families of graphs with similar local structures in the asymptotic regime.
At a high level, a sequence of graphs $\{G_n\}_{n\in\mathbb N}$ is said to converge locally if the empirical distribution of the local neighborhood of a uniform random node converges. 
The original definition of local convergence by \citet{benjamini2001} applies to graphs with no node or edge features. Here, we consider graphs with attributes where each node and edge is associated with an input feature and (in the case of supervised learning) a target feature.  This is reminiscent of `marked graph convergence'   \citep{benjamini2015unimodular} \citep[Ch. 2]{RemcoVol2}.

To formalize the definition of local convergence, let $(G,o)$ denote a rooted graph with attributes, which is a graph $G$ with node/edge attributes to which we assign a root node $o$. Let $\mathcal{G}_*$ be the set of all possible rooted graphs with attributes. A limit graph is defined as a measure over the space $\mathcal{G}_*$ with respect to a local metric $d_{loc}$. For a pair of rooted graphs $(G_1,o_1)$ and $(G_2,o_2)$, the distance $d_{loc}$ is given by $$d_{loc}((G_1,o_1),(G_2,o_2)) = \frac{1}{1+\inf_k\{k:B_k(G_1,o_1)\not\simeq B_k(G_2,o_2)\}},$$ where $B_k(G,v)$ is the $k$-hop neighborhood of node $v$ in graph $G$, and $\simeq$ represents the graph isomorphism.
Since the limit graph is rooted, we need to make finite graphs in the sequence $\{G_n\}_{n\in\mathbb N}$ rooted as well by choosing a uniform random root denoted $\mathcal{P}_n=\frac{1}{n}\sum{o_n\in V(G_n)}\delta{(G_n,o_n)}$.

\begin{definition}[Local Convergence with Attributes]\label{def: lwc}
Let $\mu$ be a measure on the space $\mathcal{ G}_*$ . 
Then, a sequence of graphs $\{G_n\}_{n\in\mathbb N}$ is said to converge locally in probability to a graph $\mu\sim\mathcal{G}_*$
if, for any $k>0$ and any finite graph $Q$ with nodes at most $k$ hops from the root,
\begin{equation*}
\mathbb{P}_{v\sim\mathcal{P}_n}[Q\sim B_k(G_n,v)]\overset{\mathbb{P}}{\to}\mathbb{P}_{\mu}(Q\sim B_k(G,o)).
\end{equation*}
Equivalently, for any bounded and continuous (with respect to metric $d_{loc}$) function $f:\mathcal{G}_*\to \mathbb R$, 
\begin{equation}\label{eq: lwc function definition}
\mathbb{E}_{v\sim\mathcal{P}_n}[f(G_n,v)|G_n]\overset{\mathbb{P}}{\to}\mathbb{E}_{\mu}[f(G,o)].
\end{equation}
 \end{definition}
The above definition points out the equivalence between the convergence of local neighborhoods around random nodes and the convergence of bounded and continuous functions known as local functions (for the proof of equivalence, see  \citep[Ch. 2]{RemcoVol2}). 
Intuitively, a local function applied to finite rooted samples of a graph limit can be shown to converge to the function applied directly to the infinite graph limit. Building upon this idea, we define \emph{almost local functions} as follows.
  
\begin{definition}[Almost Local Functions] A function $f:\mathcal{G}_*^N\to\mathbb R^K$ is said to be \textit{almost local} if, for a sequence of graphs $\{G_n\}_{n\in\mathbb N}$ converging to a limit graph $\mu\sim\mathcal{G}_*$, it converges to a limit function $\tilde f:\mathcal{G}_*^N\to\mathbb R^K$ as
% the equation \ref{eq: lwc function definition} holds, i.e.,
\begin{equation}
\mathbb{E}_{v_1, v_2 \ldots v_N\sim\mathcal{P}_n}\Big[f\big((G_n,v_1),\ldots,  (G_n,v_N)\big)|G_n\Big]\overset{\mathbb{P}}{\to}\mathbb{E}_{(G^{(i)},o_i)\sim \mu}\Big[\tilde f\big((G^{(1)},o_1),\ldots (G^{(N)},o_N)\big)\Big].
\end{equation}
\end{definition}

\textbf{Remark.} The definition above departs slightly from the conventional definition of local functions in the literature, which typically requires boundedness and continuity. Still, the introduction of almost local functions is necessary for the analysis of sampling-based GNNs.
For instance, many typical sampling methods assign weights to nodes based on global parameters of the graph such as the average moments of the degree, which violates the local continuity condition. Similarly, loss functions based on negative sampling (e.g., \citep{hamilton2017inductive}) are not local functions in the conventional sense as they depend on the embedding of vertices located far away from the root. In Appendix~\ref{appendix: GL proof}, we prove that both normalized adjacency sampling and loss functions with negative sampling yield functions that are almost local.

\subsection{GNN in the limit}\label{sec: GNN in the limit}
We define limit GNNs, i.e., GNNs on infinite graph limits, by extending the aggregate-readout architecture in \eqref{eqn:propagation} to infinite graphs.
% The idea behind defining GCNs on the limit of graphs is to extend the message-passing framework used by GCNs to the limit. 

% This enables us to define GNNs on the limit, using the propagation that updates the feature matrix in each layer by applying  \eqref{eqn:propagation} on the $L$-neighborhood of the root node, $B_L(G,o)$. 
\begin{definition}[Limit GNNs] Consider a (possibly infinite) rooted graph $(G,o)$ drawn from $\mathcal{G}*$ following distribution $\mu$.
Given a GNN with $L$ layers, consider an $L$-neighborhood of the root $o$, denoted $B_L(G,o)$. Let $\mathbf Z_o$ be the output of the $L$-layer GNN given graph $B_L(G,o)$  [cf. \eqref{eqn:propagation}]. 
Then, the output embedding of the limit GNN is $\mathbb{E}{(G,o)\sim\mu}[{\mathbf Z_o}]$.
\end{definition}

In the next section, we use this limit GNN to analyze sampling-based GNNs fitting the description of Algorithm~\ref{alg: training}. 

%% file: convergence_sampling_gcn.tex
%%%%%%%%%%%%%%%%%%%%%%%%%%%%%%%%%%%%%%%%%%%%%%%%%%%%%%%
%%%%%%%%%%%%%%%%%%%%% CONVERGENCE %%%%%%%%%%%%%%%%%%%%%
%%%%%%%%%%%%%%%%%%%%%%%%%%%%%%%%%%%%%%%%%%%%%%%%%%%%%%%

\section{Convergence of Sampling-Based GNNs }\label{sec: convergence result}

%We are now prepared to present our main result, 
% which is the convergence of training using Algorithm~\ref{alg: training}.  This result is built upon the foundation laid out in the previous sections, where we established the following key steps: 1) viewing the large input graph as a graph limit, and 2) training our GNN algorithm using a sequence of sampled subgraphs that have a similar local structure to the limit. 
% Our main result
Our main result demonstrates that a sampling-based GNN trained on a collection of subgraphs sampled from the large target graph converges to an $\epsilon$-neighborhood of optimal limit GNN, i.e., the GNN that would be obtained by training on the full graph. We give convergence results for transductive and inductive learning tasks, and show their application to commonly used GNN architectures.
Our results rely on two sets of assumptions. First, we require the loss function to be bounded and Lispchitz continuous. This assumption is not very stringent, and is commonly used in the literature.

\begin{assumption}[Loss Function]\label{assum: lipschitz}
The loss function $\mathcal L$ is bounded. Further,
the loss function $\mathcal L$ and its gradient $\nabla\mathcal L$ are Lipschitz in the learning coefficients $\mathbf W$ with Lipschitz constant $C$.
\end{assumption}

% \red{L: So there is no randomness associated with the graph sequence, correct? We are assuming a fixed graph sequence that converges in the local-weak sense.} \blue{Y: no there is randomness.}
% \todo{loss is bounded}
%\begin{assumption}[Constant Number of Layers]\label{assum: layers}
%The number of layers of the GNN in \eqref{eqn:propagation} is bounded by a fixed constant $L>0$.
%\end{assumption}

%The second set of assumptions is specific to the setting of local limits and the sampling procedure used to generate the sequence of graphs. Specifically, we assume 
Second, we require the collection of subgraphs on which the GNN is trained to satisfy the conditions of local convergence given in Section \ref{sec: GL}. We further assume that the sampling methods in Algorithm~\ref{alg: training} are almost local, and hence can be defined on the graph limit.

\begin{assumption}[Almost Local Loss, Sampler and Aggregators] \label{assum: local}
% The gradient sampling scheme $\nu_g$, the 
The computational graph sampling scheme $\nu_C$, the loss function $\mathcal L$, and  $\texttt{COMBINE}_\ell$ and $\texttt{AGGREGATE}_\ell$ for $\ell\in[0,L]$ are almost local. 
\end{assumption}
\begin{assumption}[Convergent Sequence of Graphs]\label{assum: GL}
The sequence of graphs $\{G_n\}_{n\in \mathbb N}$ converges locally in probability to $(G,o)\sim\mu$, where $\mu$ is a probability measure on the space of rooted graphs $\mathcal{G}_*$.
\end{assumption}
% {\color{orange}
% {\bf Amin:} A few comments about theorem statements. 

% 1- Can we start with a theorem that starts with a sequence of graphs G1...Gn, applies Alg1, gets and embeddings Z1...Zn and then says limit n->infty Zn = Zinfty, where Zinfty is the embedding of the limit?

% 2- Then we can follow that with a theorem (corollary?) that is G is large (infty or large enough to be eps-close) then the result of training on samples from the graph is close to training on the whole graph? 

% This way we have made the two points clearly and separately. 
% 3- Finally, I don't understand why we can't obtain a result on graph regression/classification/embedding by just adding another layer that aggregates node embeddings? If all  node embeddings are epsilon-close, the final graph embedding should be close too. Such a result will better justify training on a family of graphs. Please let me know if you think such a result is not feasible/desirable. 
% }

%We now present our main result, which provides convergence guarantees for training sampling-based GNNs on sequences of convergent graphs.  
%Since most of the sampling-based GNN use conventional SGD for computing gradient, for simplicity, we consider the case in which $\nu_g$ is uniform sampling. 

\begin{theorem}\label{thm: stochastic GNN}
Consider a sampling-based GNN with $L$ layers [cf. Algorithm~\ref{alg: training}] and with uniform node sampler $\nu_g$ (SGD), and additionally satisfying Assumptions ~\ref{assum: lipschitz}--\ref{assum: local}. Let the collection of subgraphs on which the GNN is trained (generated by sampler $\mu_S$) define a convergent graph sequence $\{G_n\}_{n\in\mathbb N}$ as in Assumption \ref{assum: GL}.
Then, there exists a learning rate $\eta>0$ such that:
 \setlist{nolistsep}
\begin{enumerate}[noitemsep]
    \item For any $\epsilon>0$, there exists $N_\epsilon>0$ such that training Algorithm~\ref{alg: training} on subgraph samples of size at least $N_\epsilon$ converges to the $\epsilon$-neighborhood of the optimal GNN on the limit $G$.
    \item The expected number of training steps, ad therefore the expected number of subgraph samples needed for convergence, is $\smash{\tilde O(\frac{1}{\epsilon^2})}$.
\end{enumerate}
\end{theorem}

Let $\mathbf W_t$ denote the GNN coefficients learned in iteration $t$ of Algorithm~\ref{alg: training}, i.e., the weights of the GNN trained on the collection of random subgraphs. Convergence to the $\epsilon$-neighborhood of the optimal limit GNN means that, after a finite number of training steps, the expected gradient of the loss in the limit graph, using these same coefficients $\mathbf W_t$, is bounded by $\epsilon$; or, explicitly, that $\mathbb E_{\mu, W_t}\big(|\nabla_{\mathbf W_t} \mathcal L(W_t, G)|\big)\leq \epsilon$. Here, note that the randomness arises from both the limit $(G,o)\sim \mu$ and the random initialization of the coefficients $\mathbf W_0$.
% In other words, the distance from the optimal learning coefficients in the limit is small.

As a special case of this theorem, we focus on the situation in which the collection of subgraphs generated by $\mu_S$ is a collection of local subgraphs. By that we mean that $\mu_S$ first samples the infinite graph $(G,o)\sim \mu$ and then returns the breadth-first local neighborhood of size $N_\epsilon$ around the sampled root. If we think of the large target graph as the graph limit, the following result states that training on small subgraphs sampled via breadth-first search (BFS) is enough to get to the $\epsilon$-neighborhood of the optimal GNN on the large graph. 

\begin{corollary}\label{cor: }
Given a sampling-based GNN satisfying the assumptions of Theorem~\ref{thm: stochastic GNN}, 
let the subgraph sampler $\mu_S=B_{N_\epsilon}(\mu(G,o)$ be a local BFS sampler. Then, the result of Theorem~\ref{thm: stochastic GNN} holds.
\end{corollary}

% \begin{theorem}\label{thm: stochastic GNN}
% Consider training a sampling-based GNN as in Algorithm~\ref{alg: training}, satisfying Assumptions ~\ref{assum: lipschitz}- \ref{assum: local} on $\{G_n\}_{n\in\mathbb N}$, on a sequence of (possibly random) convergent graphs with the limit $G\sim \mu$. Further assume that the gradient sampler $\nu_g$ is uniform, as in SGD. 
% Then, there exists a learning rate $\eta>0$ such that for any $\epsilon>0$, training Algorithm~\ref{alg: training} on the sequence $\{G_n\}_{n\in\mathbb N}$ converges to the $\epsilon$-neighborhood of the GNN learned on the limit, provided that $|V(G_n)|\geq N_\epsilon$ for some $N_\epsilon$. Further, the number of steps needed for convergence  is $\smash{\tilde O(\frac{1}{\epsilon^2})}$, i.e., if we define $t^*=\inf_t\{\mathbb{E}_{\mu,\mathbf W_t}(\nabla_{\mathbf W_t} \mathcal L(\mathbf G))\leq \epsilon\}$, then $\mathbb{E}_\mu[t^*]=O(\frac{1}{\eta\epsilon^2})$.
% \end{theorem}

% Also, note that the number of graphs needed in the input sequence is the number of training iterations, $t^*$. Since the theorem places bounds on the number of training steps required ($t^*$), it also provide bounds the number of graphs needed in the input sequence of Algorithm~\ref{alg: training}. 

This result is general and applies to various GNN architectures, including but not limited to GCN  \citep{kipf17-classifgcnn}, GraphSAGE \citep{hamilton2017inductive}, FastGCN \citep{chen2018fastgcn}, and shaDoW-GNN \citep{zeng2021decoupling}. 
\begin{corollary}\label{cor: application}
    GCN, GraphSAGE, FastGCN, and shaDoW-GNN satisfy Assumption~\ref{assum: local}. Therefore, under Assumption~\ref{assum: lipschitz}, the results of Theorems~\ref{thm: stochastic GNN} and Corollary ~\ref{cor: } hold. 
\end{corollary}

This corollary has an important practical implication: it gives guarantees allowing practitioners to compare any of these GNNs by training them on small samples from the large target graph, which is much less costly than doing so on the large graph itself. In Appendix~\ref{sec:app}, we provide detailed explanations for how each of these models fits into the unified framework of Algorithm \ref{alg: sampling-GNN} and satisfies the assumptions of our main result. 

We conclude by showing that our result also applies to graph learning or transductive graph machine learning tasks. To do so, we need an additional assumption on the \texttt{READOUT} layer in \eqref{eqn: readout}, which must be almost local. This enables extending our results to even more architectures, such asGIN \citep{xu2018GIN} with mean aggregation in the readout layer.

\begin{theorem}\label{thm: transductive}
For a transductive sampling-based GNN satisfying the assumptions of Theorem~\ref{thm: stochastic GNN} and with an almost-local \texttt{READOUT} layer, the result of Theorem~\ref{thm: stochastic GNN} holds.
\end{theorem}

%% file: previous-gnn-models.tex
%%% Applications %%%

\section{Applications}\label{sec:app}
Our main result applies to various GNN architectures, including GCN, GraphSAGE, and GIN with mean readout; and various sampling mechanisms, such as neighborhood sampling (GraphSAGE), FastGCN, and shaDoW-GNN. 
% We first state the corollary of our result for models on estimating node-embedding that do not use importance sampling in gradient estimation and then state the result for GraphSAINT.  
% Given our result, a practitioner can compare any of these frameworks by training them on small samples from the large input graph, and then compare the obtained loss.
In each of the following sections, we will discuss how each model fits into our unified framework (Algorithm~\ref{alg: training}), and explain why they satisfy the assumptions of the main theorem.

\noindent\textbf{GCN.}
The GCN architecture \citep{kipf17-classifgcnn}, as described in \eqref{eqn:gcn}, applies the convolutional layer on the entire computational graph of a node. So, for a $L$-layer GCN and a given node $v$, the computational graph sampler $\mu_C(v)$ returns the entire $L$-neighborhood of $v$. In addition, it is common to use SGD to compute the gradient in GCNs. Both the gradient and computational graph samplers satisfy the locality assumption \ref{assum: local}, and hence, our main theorem applies.

\noindent\textbf{GIN.}
In the Graph Isomorphism Network (GIN)
\citep{xu2018GIN}, the \texttt{AGGREGATE} and \texttt{COMBINE} operations consist of a multi-layer perceptron applied to the sum of each node's embeddings with their neighbors' embeddings. In its standard form, GIN thus takes in the entire computational graph of a node. Since the \texttt{AGGREGATE} and \texttt{COMBINE} operations are local, in a node-level task GIN satisfies all of our assumptions, hence our results hold. In a graph-level task, our results hold provided that the readout is a mean aggregation, which is permutation invariant and, unlike the sum, does not increase with the graph size.

% \begin{corollary}\label{cor: GNN}
% Let $\{G_n\}_{n\in\mathbb N}$ be a sequence of (possibly random) convergent graphs (Assumption~\ref{assum: GL}). Then, there exists a small enough learning rate $\eta$ such that for any $\epsilon>0$, training the full GCN as in Algorithm~\ref{alg: training} with $\nu_g$ and $\nu_c$ sampling the entire space, a constant number of layers, and Lipschitz loss function (Assumptions~\ref{assum: lipschitz} and ~\ref{assum: layers}) on the sequence of convergent graphs with learning rate $\eta$ converges to the $\epsilon$-neighborhood of learning GCN on the limit in $\smash{\tilde O\big(\frac{1}{\epsilon^2}\big)}$ expected steps.
% \end{corollary}

\noindent\textbf{GraphSAGE with neighborhood sampling.}
GraphSAGE \citep{hamilton2017inductive} is a popular GNN architecture that generates node embeddings by concatenating information from each node's local neighborhood. They also propose to train GraphSAGE with a computational graph sampling technique called neighborhood sampling, where they sample a fixed number of neighbors for each node. The computational graph sampler $\nu_C$ assigns probabilities proportional to $\smash{1/{deg(v)\choose K_\ell }}$ to all sets of size $K_\ell$ from the neighbors of node $v$ at layer $\ell$, where $K_\ell$ is the number of nodes to sample at layer $\ell$ and $deg(v)$ is the degree of node $v$. They also use SGD for gradient sampler $\nu_g$.

% Our framework applies to GraphSage to bound the number of samples it needs.
Another novelty of the GraphSAGE approach was to suggest unsupervised learning based on computing loss with negative sampling, which is almost local per Proposition~\ref{prop: almost local fnc}. Therefore, our result applies to both their semi-supervised and unsupervised training.

% \begin{corollary}
% Fix some $\epsilon>0$. Then, there exists some $N_\epsilon$ such that training a GraphSAGE with loss function satisfying the assumptions in Theorem~\ref{thm: stochastic GNN} on the convergent sequence $\{G_n\}_{n\in\mathbb N}$ with $|V(G_n)|\geq N\epsilon$ converges to the $\epsilon$-neighborhood of learning GraphSAGE on the limit in $\smash{O\big(\frac{1}{\eta\epsilon^2}\big)}$ expected number of steps.
% \end{corollary}

\noindent\textbf{FastGCN.}
Fast GCN \citep{chen2018fastgcn} relies on layerwise sampling to address  scalability issues GNNs. %A pioneering method in this regard is FastGCN \cite{chen2018fastgcn}, which tackles the neighborhood expansion problem by introducing a layer-specific importance sampling scheme. 
The computational graph sampler in FastGCN subsamples nodes from each layer based on the normalized adjacency matrix, as expressed by the equation
\[q(u;v)=\frac{\|\mathbf A'(u,v)\|^2}{\|\sum_{u'\in N(v)}\mathbf A'(u',v)\|^2},\]
which by Proposition~\ref{prop: almost local fnc} is almost local.
So, for a node $v$ already sampled in the computational graph, $\nu_C$ samples a fixed number of nodes $k_\ell$ in the next layer w.r.t $q(u;v)$.
% This method has a smaller variance than GraphSAGE, and \cite{chen2018fastgcn} showed that when the number of samples goes to infinity then SGD converges.
This architecture also uses SGD as the gradient sampler. %Therefore, this framework satisfies the assumption of our main result.

% \begin{corollary}
% Fix some $\epsilon>0$. Then there exists some $N_\epsilon$ such that training a FastGCN with loss function satisfying the assumptions in Theorem~\ref{thm: stochastic GNN} on the convergent sequence $\{G_n\}_{n\in\mathbb N}$ with $|V(G_n)|\geq N\epsilon$ converges to the $\epsilon$-neighborhood of learning a FastGCN on the limit in $\smash{O\big(\frac{1}{\eta\epsilon^2}\big)}$ expected number of steps.
% \end{corollary}

\noindent\textbf{shaDoW-GNN.}
shaDoW-GNN \citep{zeng2021decoupling} is a method that tackles the scalability issue of GNNs by subsampling a subgraph for each node in a minibatch. Specifically, for a $L'$-layer GNN,  shaDoW-GNN selects a subgraph with nodes that are at most $L$ hops away from each node in the minibatch, where $L \leq L'$. They decouple nodes that appear in more than one sampled subgraph by keeping two copies of them. Their framework allows either to  keep the whole $L$-neighborhood or to sample from it, similar to GraphSAGE or FastGCN. So, the computational graph sampler $\nu_C$ is similar to one of the previous frameworks, with the only difference being that it creates new copies in memory for each node sampled multiple times.
% Once the subgraph is selected, they run full GCN on each subgraph. 

\noindent\textbf{GNNAutoScale.} A recent addition to the family of GNNs is GNNAutoScale\citep{fey2021gnnautoscale}, which integrates well with our framework. It leverages historical embeddings of out-of-sample neighbors during training, merging minibatch sampling and historical embeddings.  This method aligns well with our framework for the following reason: Each iteration of embedding calculation is convergent, as shown by Lemma~\ref{lm: nabla bound-stochastic}. So, one can use historical embeddings as `features’ for subsequent iterations during information aggregation, so it satisfies the Assumption~\ref{assum: local}. Therefore, we can view GNNAutoScale's approach through the lens of our framework.

\noindent\textbf{GraphFM.} This framework by \citep{yu2022graphfm}, although bearing similarities with GNNAutoScale, distinguishes itself by including historical embeddings of nodes within the one-hop boundary of selected minibatches. This novel incremental update strategy remains consistent with the principles of our framework, indicating a promising compatibility.

\noindent\textbf{LMC.} Another model, LMC\citep{shi2022lmc}, while echoing GNNAuto-scale in many respects, shows unique differences in aggregation during its forward/backward propagation. Its localized aggregation strategies remain consistent with our proposed framework, satisfying with the assumptions of our result.

\noindent\textbf{IBMB.} This method by \citep{gasteiger2022IBMB} employs a distinctive approach by computing influence scores for nodes and subsequently optimizing the selection of influential nodes for computation. Despite its broad methodology of calculating node influence, the implementation specifics, especially using pagerank computation and localized node selection, align well with our framework (given that pagerank is a local function as shown in \citep{garavaglia2020local}).